\documentclass{article}

\usepackage{microtype}
\usepackage{graphicx}
\usepackage{subfigure}
\usepackage{booktabs} 
\usepackage{amsmath,bm}
\usepackage{enumitem}

\usepackage{hyperref}

\usepackage[accepted]{icml2024}

\usepackage{amsmath}
\usepackage{amssymb}
\usepackage{mathtools}
\usepackage{amsthm}
\newcommand{\fdbd}{\texttt{f}\texttt{DBD}}
\usepackage[capitalize,noabbrev]{cleveref}

\theoremstyle{plain}
\newtheorem{theorem}{Theorem}[section]
\newtheorem{proposition}[theorem]{Proposition}

\theoremstyle{definition}
\newtheorem{definition}[theorem]{Definition}

\theoremstyle{remark}

\newcommand{\<}{\langle}
\renewcommand{\>}{\rangle}
\newcommand{\T}[1]{\mathrm{#1}}

\usepackage[textsize=tiny]{todonotes}

\icmltitlerunning{Fast Decision Boundary based Out-of-Distribution Detector}

\begin{document}

\twocolumn[
\icmltitle{Fast Decision Boundary based Out-of-Distribution Detector}




\begin{icmlauthorlist}
\icmlauthor{Litian Liu}{mit}
\icmlauthor{Yao Qin}{ucsb}
\end{icmlauthorlist}

\icmlaffiliation{mit}{MIT}
\icmlaffiliation{ucsb}{UC Santa Barbara}

\icmlcorrespondingauthor{Litian Liu}{litianl@mit.edu}
\icmlcorrespondingauthor{Yao Qin}{yaoqin@ucsb.edu}

\icmlkeywords{Machine Learning, ICML}

\vskip 0.3in
]



\printAffiliationsAndNotice{}  


\begin{abstract}
Efficient and effective Out-of-Distribution (OOD) detection is essential for the safe deployment of AI systems. Existing feature space methods, while effective, often incur significant computational overhead due to their reliance on auxiliary models built from training features. In this paper, we propose a computationally-efficient OOD detector without using auxiliary models while still leveraging the rich information embedded in the feature space.
Specifically, we detect OOD samples based on their feature distances to decision boundaries. To minimize computational cost, we introduce an efficient closed-form estimation, analytically proven to tightly lower bound the distance. 
Based on our estimation, we discover that In-Distribution (ID) features tend to be further from decision boundaries than OOD features. 
Additionally, ID and OOD samples are better separated when compared at equal deviation levels from the mean of training features.
By regularizing the distances to decision boundaries based on feature deviation from the mean, we develop a hyperparameter-free, auxiliary model-free OOD detector. 
Our method matches or surpasses the effectiveness of state-of-the-art methods in extensive experiments while incurring negligible overhead in inference latency. Overall, our approach significantly improves the efficiency-effectiveness trade-off in OOD detection.
Code is available at: \url{https://github.com/litianliu/fDBD-OOD}.
\end{abstract}

\vspace{-3mm}

\section{Introduction}
\label{sec:introduction}

\begin{figure}
\vspace{-2mm}
\begin{center}
\includegraphics[width=\columnwidth]{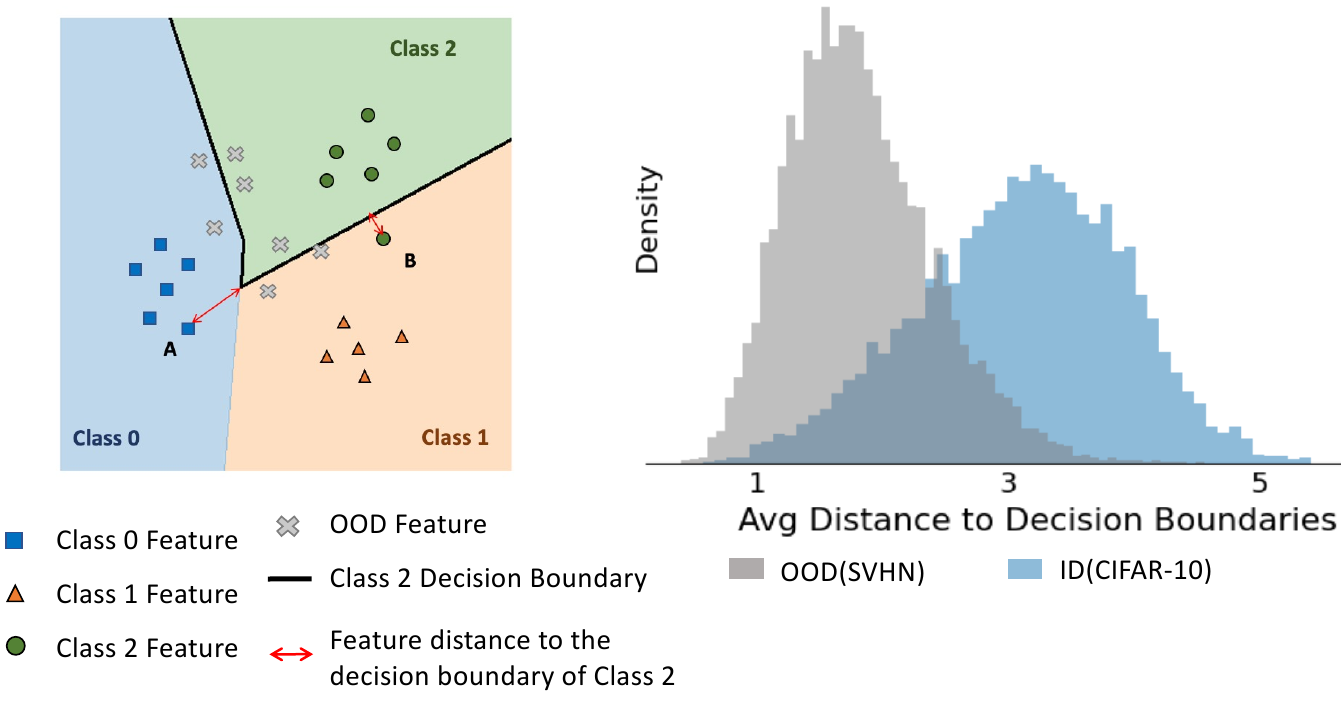}
\end{center} 
\vspace{-5mm}
\caption{
\textbf{Overview.}
\textit{Left: Conceptual Illustration.} The feature distance to decision boundaries on a multi-class classifier's penultimate layer, quantifying the perturbation magnitude needed to alter the model prediction to a class 
(see formal definition in Section~\ref{sec:dist-measure}). 
\textit{Right: Empirical Observation.} 
Features of ID samples (CIFAR-10) tend to reside further from decision boundaries than OOD samples (SVHN).  
The distances are measured using our method (see Section~\ref{sec:dist-measure}) and averages are per sample. 
}
\label{fig:overview}
\vspace{-5.5mm}
\end{figure}

As machine learning models are increasingly deployed in the real world, it is inevitable to encounter samples out of the training distribution.
Since a classifier cannot make meaningful predictions on test samples from classes unseen during training, the detection of Out-of-Distribution (OOD) samples is crucial for taking necessary 
precautions. 
The field of OOD detection, which has recently seen a surge in research interest \cite{yang2021generalized}, divides into two main areas. 
One area investigates the training time regularization to enhance OOD detection \cite{wei2022mitigating, huang2021mos, ming2022exploit}, 
while our work, along with others, delves into \emph{post-hoc} methods, which are training-agnostic and suitable for ready implementation on pre-trained models. 
OOD detectors can be
designed over model output space \cite{liang2018enhancing, liu2020energy, hendrycks2019scaling}.
Additionally, \citet{tack2020csi, lee2018simple, sun2022out} and \citet{sastry2020detecting} use the clustering of In-Distribution (ID) samples in the feature space for OOD detection. 
For example, \citet{lee2018simple} fit a multivariate Gaussian over the training features and detect OOD based on the Mahalanobis distance, and \citet{sun2022out} detect OOD based on the k-th nearest neighbor (KNN) distance to the training features.
While existing feature-space methods are highly effective, their reliance on auxiliary models built from training features incurs additional computational costs.
This poses a challenge for time-critical real-world applications, such as autonomous driving, where the latency of OOD detection becomes a top priority.

In this work, we focus on designing post-hoc OOD detectors for pre-trained classifiers.  
We aim to leverage the rich information in the feature space while optimizing computational efficiency and avoiding the need for auxiliary models built from training statistics.
To this end, we study from the novel perspective of decision boundaries, which naturally summarizes the training statistics.
We begin by asking:
\vspace{-2mm}
\begin{center}
\emph{Where do features of ID and OOD samples reside with respect to the decision boundaries?}
\end{center}
\vspace{-2mm}
To answer the question, we first formalize the concept of the feature distance to a class's decision boundary.
We define the distance as the minimum perturbation in the feature space to change the classifier's decision to the class, visually explained in Figure~\ref{fig:overview}~\textit{Left}.
In particular, we focus on the penultimate layer, \emph{i.e.}, the layer before the linear classification head. 
Due to non-convexity, the distance on the penultimate layer cannot be readily computed. 
To minimize the cost of measuring the distance, we introduce in Section~\ref{sec:dist-measure} an efficient closed-form estimation, analytically proven to tightly lower bound the distance.
Intuitively, feature distances to decision boundaries reflects the difficulty of changing model decisions and can quantify model uncertainty in the feature space.
Unlike output space \emph{softmax} confidence, our feature-space distance uses the rich information embedded in the feature space for OOD detection.

Based on our closed-form distance estimation, we pioneeringly explore OOD detection from the perspective of decision boundaries.
Intuitively, features of ID samples would reside further away from the decision boundaries than OOD samples, since a classifier is likely to be more decisive in ID samples. 
We empirically validate our intuition in Figure~\ref{fig:overview} (\textit{Right}).
Further, we observe that ID and OOD can be better separated when compared at equal deviation levels from the mean of training features. 
Using the deviation level as a regularizer, we design our detection score as a regularized average feature distance to decision boundaries.
The lower the score is, the closer the feature is to decision boundaries, and the more likely the sample is OOD.

Thresholding on the detection scores, we have \textbf{f}ast \textbf{D}ecision \textbf{B}oundary based OOD \textbf{D}etector (\fdbd).
Our detector is hyperparameter-free and auxiliary model-free, eliminating the cost of tuning parameters and reducing the inference overhead.
Moreover, \fdbd~scales linearly 
with the number of classes and the feature dimension, theoretically guaranteed to be computationally scalable for large-scale tasks.
In addition, \fdbd~incooperates class-specific information from the class decision boundary perspective to improve OOD detection effectiveness. 

With extensive experiments, we demonstrate the superior efficiency and effectiveness of our method across various OOD benchmarks on different classification tasks (ImageNet~\cite{deng2009imagenet}, CIFAR-10~\citep{krizhevsky2009learning}), diverse training objectives (cross-entropy $\&$ supervised contrastive loss \citep{khosla2020supervised}), and a range of network architectures (ResNet~\cite{he2016deep} $\&$ ViT~\cite{dosovitskiy2020image} $\&$ DenseNet~\citep{huang2017densely}).
Notably, our \fdbd~consistently achieves or surpasses state-of-the-art OOD detection performance.
In the meantime, \fdbd~maintains inference latency comparable to the vanilla \emph{softmax-confidence} detector, inducing practically negligible overhead in inference latency. 
Overall, our method significantly improves upon the efficiency-effectiveness trade-off of existing methods. We summarize our main contributions below:
\vspace{-3mm}

\begin{itemize}
    \item \textbf{Closed-form Estimation of the Feature Distance to Decision Boundaries
    } 
    In Section~\ref{sec:dist-measure}, we formalize the concept of the feature distance to decision boundaries. 
    We introduce an efficient and effective closed-form estimation method to measure the distance, providing a beneficial tool for the community.  
    \vspace{-2mm}
    
    \item \textbf{Fast Decision Boundary based OOD Detector:}
    Using our estimation method in Section~\ref{sec:dist-measure}, we establish in Section~\ref{sec:fdbd} the first empirical observation that ID features tend to reside further from decision boundaries than OOD features.
    This ID/OOD separation is enhanced when regularized by the feature deviation from the training feature mean. 
    Based on the observation, we propose a hyperparameter-free, auxiliary model-free, and computationally efficient OOD detector from the novel perspective of decision boundaries.
    \vspace{-5.5mm}
    
    \item \textbf{Experimental analysis:}
    In Section~\ref{sec:experiments},
    we demonstrate across extensive experiments that fDBD achieves or surpasses the state-of-the-art OOD detection effectiveness with negligible latency overhead.
    \vspace{-2mm}

    \item \textbf{Theoretical analysis:}
    We theoretically guarantee the computational efficiency of \fdbd~ through complexity analysis. 
    Additionally, we support the effectiveness of our \fdbd~ through theoretical analysis in Section~\ref{sec:justification}.
\end{itemize}

\vspace{-4mm}
\section{Problem Setting}
\vspace{-1mm}
We consider a data space $\mathcal{X}$, a class set $\mathcal{C}$, and a classifier $f:\mathcal{X} \rightarrow \mathcal{C}$, which is trained on samples \emph{i.i.d.} drawn from joint distribution $\mathbb{P}_{\mathcal{X}\mathcal{C}}$.
We denote the marginal distribution of $\mathbb{P}_{\mathcal{X}\mathcal{C}}$ on $\mathcal{X}$ as $\mathbb{P}^{in}$. 
And we refer to samples drawn from $\mathbb{P}^{in}$ as In-Distribution (ID) samples. 
In practice, the classifier $f$ may encounter $\bm{x} \in \mathcal{X}$ which is not drawn from $\mathbb{P}^{in}$.
We say such samples are Out-of-Distribution (OOD).

Since a classifier cannot make meaningful predictions on OOD samples from classes unseen during training, it is important to distinguish between such OOD samples and ID samples for deployment reliability. 
Additionally, for time-critical applications, it is crucial to detect OOD samples promptly to take precautions. 
Instead of using the clustering of ID features and building auxiliary models as in prior art \cite{lee2018simple, sun2022out}, we alternatively investigate OOD-ness from the perspective of decision boundaries, which inherently captures the training ID statistics. 

\section{Detecting OOD using Decision Boundaries}\label{sec:method}

To understand the potential of detecting OOD from the decision boundaries perspective, we ask:
\vspace{-2mm}
\begin{center}
\textit{Where do features of ID and OOD samples reside with respect to the decision boundaries?}
\end{center}
\vspace{-2mm}
To this end, we first define the feature distance to decision boundaries in a multi-class classifier.
We then introduce an efficient and effective method for measuring the distance using closed-form estimation. 
Using our method, we observe that the ID features tend to reside further away from the decision boundaries. 
Accordingly, we propose a decision boundary-based OOD detector.
Our detector is post-hoc and can be built on top of any pre-trained classifiers, agnostic to model architecture, training procedure, and OOD types.
In addition, our detector is hyperparameter-free, auxiliary model-free, and computationally efficient. 

\subsection{Measuring Feature Distance to Decision Boundaries}\label{sec:dist-measure}


We now formalize the concept of the feature distance to the decision boundaries. 
We denote the last layer function of $f$ as $f_{-1}: \mathcal{Z} \rightarrow \mathcal{C}$, which maps a penultimate feature vector $\bm{z}$ into a class $c$.
Since $f_{-1}$ is linear, we can express $f_{-1}$ as: 
\vspace{-1mm}
\begin{equation*}
    f_{-1}(\bm{z}) = \arg \max_{c \in \mathcal{C}} \bm{w}_c^T\bm{z} + b_c, 
\end{equation*}
\vspace{-7mm}

\noindent where $\bm{w}_c$ and $b_c$ are parameters 
corresponding to class $c$. 

\begin{definition}\label{def:uniDistance}
On the penultimate space of classifier $f$, we define the $L2$-distance of feature embedding $\bm{z_x}$ for sample $\bm{x}$ to the decision boundary of class $c$, where $c \neq f(\bm{x})$, as: 
\vspace{-2mm}
\begin{equation*}
    D_f(\bm{z_x}, c) = \inf_{ \{ \bm{z'}: f_{-1}(\bm{z'}) = c \} } \left\lVert \bm{z_x} -\bm{z'}  \right\rVert_2.
\vspace{-4mm}
\end{equation*}
\end{definition}
Here, $\{ \bm{z'}: f_{-1}(\bm{z'}) = c \}$ is the decision region of class $c$ in the penultimate space.
Therefore, the distance we defined is the minimum perturbation required to change the model's decision to class $c$.
Intuitively, the metric quantifies the difficulty of altering the model's decision.

As the decision region is \textit{non-convex} in general as shown in Figure~\ref{fig:overview}, the feature distance to a decision boundary in Definition~\ref{def:uniDistance} does not have a closed-form solution and cannot be readily computed. 
To circumvent computationally expensive iterative estimation, we relax the decision region and propose an efficient and effective estimation method for measuring the distance. 

\begin{theorem}\label{thm:1}
On the penultimate space of classifier $f$, the $L2$-distance between feature embedding $\bm{z_x}$ of sample $\bm{x}$ and the decision boundary of class $c$, where $c \neq f(\bm{x})$, i.e. $D_{f}(\bm{z_x},c)$, is tightly lower bounded by 
\vspace{-2mm}
\begin{equation}\label{eq:closedForm}
\Tilde{D}_f(\bm{z_x}, c) \coloneqq \frac{|(\bm{w}_{f(\bm{x})} - \bm{w}_c)^T \bm{z_x} + (b_{f(\bm{x})} - b_c)|}{\left\lVert \bm{w}_{f(\bm{x})} - \bm{w}_{c} \right\rVert_2},
\vspace{-3mm}
\end{equation}
where $\bm{z_x}$ is the penultimate space feature embedding of $\bm{x}$ under classifier $f$,  $\bm{w}_{f(\bm{x})}$ and $b_{f(\bm{x})}$ are parameters of the linear classifier corresponding to the predicted class  $f(\bm{x})$. 
\vspace{-2mm}
\begin{proof}
For any class $c$, $c \neq f(x)$, let
\vspace{-2mm}
\begin{equation*}
\begin{aligned}
& & \mathcal{Z}_c := &  \{ \bm{z}: f_{-1}(\bm{z'}) = c \} \\
& & = &  \{ \bm{z}: \bm{w}_c^T\bm{z} + b_c >  \bm{w}_{c'}^T\bm{z} + b_{c'} \ \forall c' \neq c\}; \\
& & \mathcal{Z}_c' := &  \{ \bm{z'}: \bm{w}_c^T\bm{z'} + b_c >  \bm{w}_{f(\bm{x})}^T\bm{z'} + b_{f(\bm{x})} \}. 
\vspace{-5mm}
\end{aligned}
\end{equation*}
\vspace{-1mm}
Observe that $\mathcal{Z}_c \subseteq \mathcal{Z}_c'$. 
Therefore, we have
\begin{equation} \label{eq:proof_1}
   D_{f}(\bm{z_x}, c)  = \inf_{\bm{z} \in \mathcal{Z}_c} \left\lVert \bm{z} - \bm{z_x} \right\rVert_2 \geq \inf_{\bm{z'} \in \mathcal{Z}_c'} \left\lVert \bm{z'} - \bm{z_x} \right\rVert_2. 
\vspace{-2mm}
\end{equation}
Note that geometrically $\inf_{\bm{z'} \in \mathcal{Z}_c'} \left\lVert \bm{z'} - \bm{z_x} \right\rVert_2$ represents the $l_2$ distance from $\bm{z_x}$ to hyperplane
\vspace{-1mm}
\begin{equation}\label{eq:hyperplane}
(\bm{w}_{f(\bm{x})} - \bm{w}_c)^T \bm{z} + (b_{f(\bm{x})} - b_c) = 0,
\vspace{-5mm}
\end{equation}

\noindent and thus
\vspace{-3mm}
\begin{equation} \label{eq:proof_2}
\inf_{\bm{z'} \in \mathcal{Z}_c'} \left\lVert \bm{z'} - \bm{z_x} \right\rVert_2 = \frac{|(\bm{w}_{f(\bm{x})} - \bm{w}_c)^T \bm{z_x} + (b_{f(\bm{x})} - b_c)|}{\left\lVert \bm{w}_{f(\bm{x})} - \bm{w}_{c} \right\rVert_2}.
\vspace{-2mm}
\end{equation}
\vspace{-4mm}

\noindent Combining Eqn.~\eqref{eq:proof_2} with Eqn.~\eqref{eq:proof_1}, we conclude that Eqn.~\eqref{eq:closedForm} lower bounds $D_{f}(\bm{z_x},c)$. 

We now show that equality in Eqn.~\eqref{eq:proof_1} holds for class $c_2$, corresponding to the nearest hyperplane to the sample embedding $\bm{z_x}$, i.e., 
\vspace{-3mm}
\begin{equation}\label{eq:c-def}
    c_2 := \arg \min_{c \in \mathcal{C}, c \neq f(\bm{x})} \inf_{\bm{z'} \in \mathcal{Z}_c'} \left\lVert \bm{z'} - \bm{z_x} \right\rVert_2.
\vspace{-1.5mm}
\end{equation}
Let the projection of $\bm{z}_{\bm{x}}$ on the nearest hyperplane be $\bm{p}_{\bm{x}}$.
From Eqn.~\eqref{eq:c-def}, for all $c \notin \{c_2,  f(\bm{x})\}$, we have 
\begin{equation}\label{eq:p-def}
\left\lVert \bm{p}_{\bm{x}} - \bm{z_x} \right\rVert_2
= \inf_{\bm{z'} \in \mathcal{Z}_{c_2}'} \left\lVert \bm{z'} - \bm{z_x} \right\rVert_2 
\leq
\inf_{\bm{z'} \in \mathcal{Z}_c'} \left\lVert \bm{z'} - \bm{z_x} \right\rVert_2. 
\end{equation}
Consequently, we have $\bm{p_x} \in {\mathcal{Z}'_{c}}^\complement$, i.e. $\bm{w}_c^T\bm{p_x} + b_c \leq \bm{w}_{f(\bm{x})}^T\bm{p_x} + b_{f(\bm{x})}$ for any $c \notin \{f(\bm{x}), c_2\}$. 
Intuitively, as all other hyperplanes are further away from $\bm{z_x}$ than $\bm{p_x}$, $\bm{p_x}$ and $\bm{z_x}$ must fall on the same side of each hyperplane. 
Therefore, $p_{\bm{x}}$ falls within the closure of $\mathcal{Z}_{c_2}$, i.e. $p_{\bm{x}} \in \overline{\mathcal{Z}}_{c_2}$.
It follows that
\vspace{-3mm}
\begin{equation}\label{eq:upper-bound}
\left\lVert \bm{p_x} - \bm{z_x} \right\rVert_2 \geq\inf_{\bm{z} \in \mathcal{Z}_{c_2}} \left\lVert \bm{z} - \bm{z_x} \right\rVert_2.
\end{equation}
\vspace{-8mm}

\noindent  Combining Eqn.~\eqref{eq:p-def} and Eqn.~\eqref{eq:upper-bound},  we see that equality holds in Eqn.~\eqref{eq:proof_1} for $c = c_2$. 
Therefore, we conclude that Eqn.~\eqref{eq:closedForm} tightly lower bounds $D_{f}(\bm{z_x},c)$
\end{proof}
\end{theorem}

\noindent\textbf{Effectiveness of Distance Measure}
Our Theorem~\ref{thm:1} analytically guarantees the effectiveness of our method. 
In addition, we empirically validate that our estimation method achieves high precision with a relative error of less than $1.5\%$.
See details in Appendix~\ref{app:measuring_method}. 

\noindent\textbf{Efficiency of Distance Measure}
Analytically, Eqn.~\eqref{eq:closedForm} can be computed in constant time on top of the inference process. 
Specifically, the numerator in Eqn.~\eqref{eq:closedForm} calculates the absolute difference between corresponding logits generated during model inference. 
And the denominator takes a finite number of $|\mathcal{C}|\times(|\mathcal{C}| - 1$) possible values, which can be pre-computed and 
retrieved in constant time during inference.
Empirically, our method incurs negligible inference overhead. 
In particular, on a Tesla T4 GPU, the average inference time on the CIFAR-10 classifier is 0.53ms per image \emph{with} or \emph{without} computing the distance using our method. 
In contrast, the alternative way of estimating the distance through iterative optimization takes 992.2ms under the same setup. 
This empirically validates the efficiency of our proposed estimation.  
See details in Appendix~\ref{app:measuring_method}. 
\begin{figure}
\begin{center}
\vspace{-1mm}
\includegraphics[width=\columnwidth]{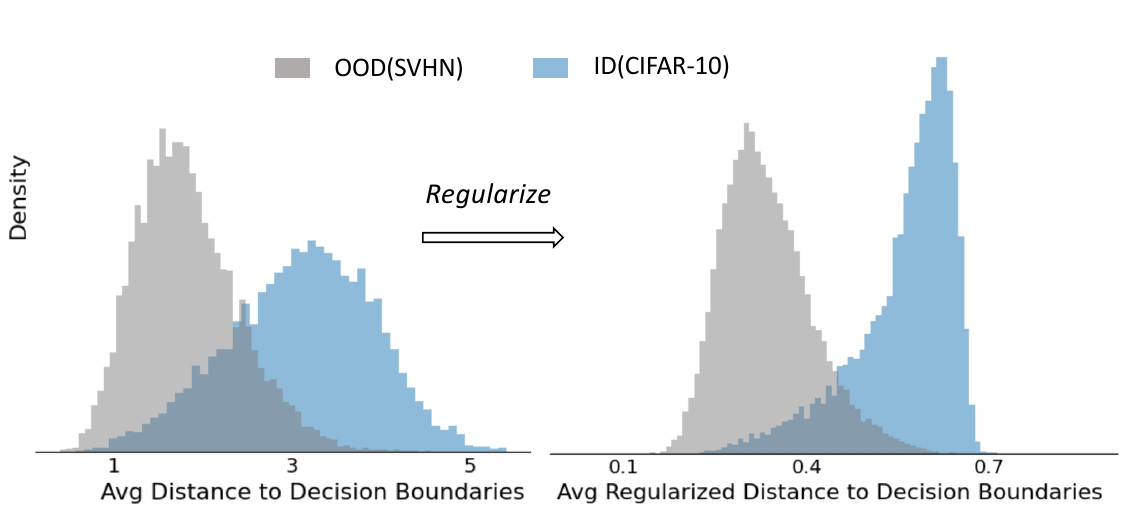}
\end{center} 
\vspace{-2mm}
\caption{
\textbf{Regularization enhances ID/OOD separation.} 
\textit{Left:} Histograms of ID/OOD features based on the average distance to decision boundaries.
\textit{Right:} Histograms of ID/OOD features based on the \emph{regularized} average distance to decision boundaries, 
which effectively compares ID and OOD features at equal deviation levels from the mean of training features.
}
\vspace{-1mm}
\label{fig:regularization}
\end{figure}

For the rest of the paper, we use our closed-form estimation in Eqn.~\eqref{eq:closedForm} to empirically study the relation between OOD-ness and the feature distance to decision boundaries, and to design our OOD detector. 

\subsection{Fast Decision Boundary based OOD Detector}\label{sec:fdbd}

We now study OOD detection from the perspective of decision boundaries.  
Recall that the feature distance to decision boundaries measures the minimum perturbation required to change the classification result.
Intuitively, the distance reflects the difficulty of changing the model's decision. 
Given that a model tends to be more certain on ID samples, we hypothesize that ID features are more likely to reside further away from the decision boundaries compared to OOD features.
We extensively validate our hypothesis in Appendix~\ref{app:individual} with plots showing ID/OOD feature distance to decision boundaries.
And we spotlight our empirical study by visualizing the per-sample average feature distance to decision boundaries for ID/OOD set in Figure~\ref{fig:regularization}~(\textit{Left}). 

Going one step further, we investigate the overlapping region of ID/OOD under the metric of the average distance to decision boundaries. 
To this end, we present Figure~\ref{fig:intuition}, where we group ID and OOD samples into buckets based on their deviation levels from the mean of training features. 
For each group, we plot the mean and variance of the average distance to decision boundaries. 
Examining Figure~\ref{fig:intuition}, we discover that the average feature distance to decision boundaries of both ID and OOD samples increases as features deviate from the mean of training features. 
We provide theoretical insights into this observation in Section~\ref{sec:justification}.
Consequently, OOD samples with a higher deviation level cannot be well distinguished from ID samples that fall into a lower deviation level. 
In contrast, within the same deviation level, OOD can be much better separated from ID samples.

Based on the understanding, we design our OOD detection score as the average feature distances to decision boundaries, \emph{regularized} by the feature distance to the mean of training features:
\vspace{-1mm}
\begin{equation}\label{eq:regDBScore}
    \mathtt{regDistDB} \coloneqq 
    \frac{1}{|\mathcal{C}| - 1} \sum_{\substack{c \in \mathcal{C}}, \ c \neq f(\bm{x}) } \frac{\Tilde{D}_f(\bm{z_x}, c)}{\| \bm{z_x} - \bm{\mu}_{train} \|_2},
\vspace{-1mm}
\end{equation}
where $\Tilde{D}_f(\bm{z_x}, c)$ is the estimated distance defined in Eqn.~\eqref{eq:closedForm} and $\bm{\mu}_{train}$ denotes the mean of training features. 
The score approximately compares ID and OOD samples at the same deviation levels. 
As demonstrated in Figure~\ref{fig:regularization}, the regularized distance score enhances the ID/OOD separation, which we explain theoretically in Appendix~\ref{app: regularization}. 
By applying a threshold on $\mathtt{regDistDB}$, we introduce the \textbf{f}ast \textbf{D}ecision \textbf{B}oundary based OOD \textbf{D}etector (\fdbd), which identifies samples below the threshold as OOD. 

\begin{figure}
\begin{center}
\includegraphics[width=0.8\columnwidth]{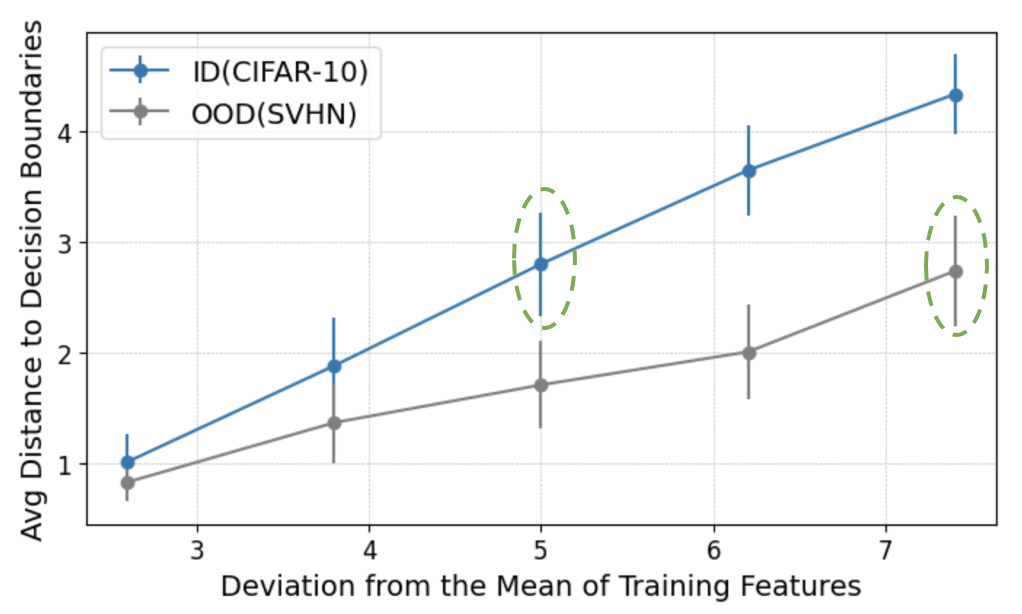}
\vspace{-5mm}
\end{center} 
\caption{
\textbf{ID and OOD are better separated at Equal Deviation Levels.} 
Features are grouped by deviation levels with group mean and variance displayed. Since the average feature distance to decision boundaries increases as features deviate from the mean of training features, 
the \emph{circled} ID/OOD groups cannot be distinguished based on their average distance to decision boundaries while being effectively separable at their own deviation levels. }
\label{fig:intuition}
\vspace{-2mm}
\end{figure}

\begin{table*}[t!]
\footnotesize
\caption{\textbf{fDBD achieves superior performance with negligible latency overhead on CIFAR-10 OOD benchmarks.}
Evaluated on ResNet-18 with FPR95, AUROC, and inference latency. 
$\uparrow$ indicates that larger values are better and vice versa.
Best performance highlighted in \textbf{bold}.
Methods with $*$ are hyperparameter-free.  
}
\vspace{-1.8mm}
\label{tab:cifar10-resnet}
\begin{center}
\includegraphics[width=0.96\textwidth]{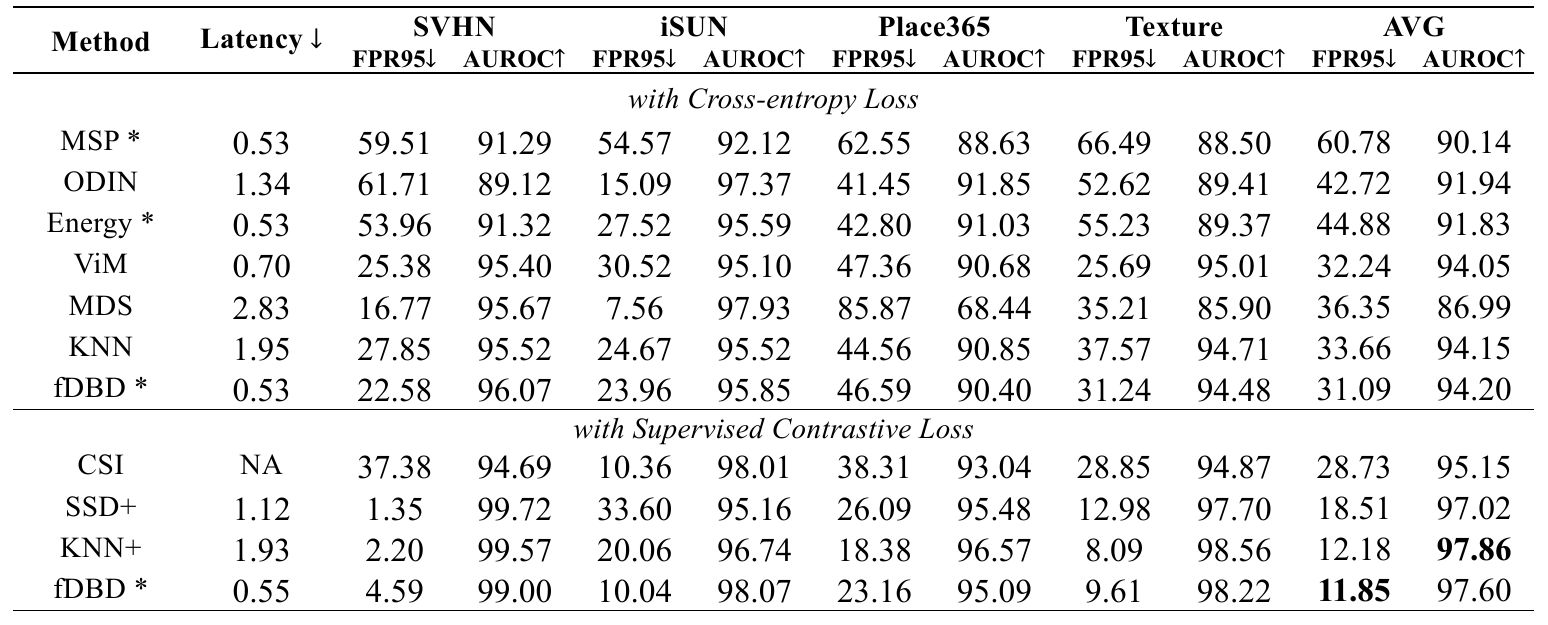}
\end{center}
\vspace{-3mm}
\end{table*}

It's worth noticing that our \fdbd~is \emph{hyperparameter-free} and \emph{auxiliary-model-free}.
In contrast to many existing approaches \cite{liang2018enhancing, lee2018simple, sun2022out}, our \fdbd~ eliminates the pre-inference cost of tuning hyper-parameter and the potential requirement for additional data.
Benefiting from our closed-form distance measuring method, \fdbd~is computationally efficient. 
Specifically, computing $\Tilde{D}_f(\bm{z_x}, c)$ takes constant time (Section~\ref{sec:dist-measure}) and computing $\|\bm{z_x} - \bm{\mu_G} \|_2$ in Equation~\ref{eq:regDBScore} has time complexity $O(P)$, where $P$ is the dimension of penultimate layer. 
Overall, \fdbd~has time complexity $O(|\mathcal{C}| + P)$, which scales linearly with the number of training classes $|\mathcal{C}|$ and the dimension $P$, indicating
computational scalability
for larger datasets and models. 
We will further demonstrate the efficiency of \fdbd~through experiments in Section~\ref{sec:experiments}.

\section{Experiments}\label{sec:experiments}

In this section, we demonstrate the superior efficiency and effectiveness of $\fdbd$ across OOD benchmarks.
We use two widely recognized metrics in the literature: the False Positive rate at $95\%$ true positive rate (FPR95) and the Area Under the Receiver Operating Characteristic Curve (AUROC). 
A lower FPR95 score indicates better performance, whereas a higher AUROC value indicates better performance. 
In addition, we report the per-image inference latency (in milliseconds) evaluated on a Tesla T4 GPU.
We refer readers to Appendix~\ref{app:imp-detail} for implementation details.

\subsection{Evaluation on CIFAR-10 Benchmarks}\label{sec:cifar}

In Table~\ref{tab:cifar10-resnet}\footnote{$\mathtt{CSI}$ results copied from Table 4 in \citet{sun2022out}.}, we present the evaluation of baselines and our \fdbd~across CIFAR-10 OOD benchmarks on ResNet-18.

\textbf{Training Schemes}
We evaluate OOD detection performance on a model trained under the standard cross-entropy loss, achieving an accuracy of $94.21\%$. 
Moreover, we experiment with a model whose representation mapping is trained using supervised contrastive loss (SupCon) \cite{khosla2020supervised}. 
With a linear classifier trained on top of the representation mapping, the model achieves an accuracy of 94.64\%. 
We note that classifiers trained with SupCon loss reach competitive accuracy,
making them essential for real-world deployment and highlighting the importance of studying OOD detection performance on such models.
As shown by \citet{sun2022out}, clustering-based OOD detectors excel for models trained with SupCon loss.
Thus, we aim to assess if \fdbd~can achieve state-of-art performance in such competitive scenarios.

\textbf{Datasets} 
On the CIFAR-10 OOD benchmark, we use the standard CIFAR-10 test set with 10,000 images as ID test samples. 
For OOD samples, we consider common OOD benchmarks: SVHN~\cite{netzer2011reading}, iSUN~\cite{xu2015turkergaze}, Places365~\cite{zhou2017places}, and Texture~\cite{cimpoi2014describing}. 
All images are of size $32 \times 32$.

\textbf{Baselines}
We compare our method with six baseline methods on the model trained with standard cross-entropy loss. 
In particular, MSP \cite{hendrycks2016baseline}, ODIN \cite{liang2018enhancing}, and Energy \cite{liu2020energy} design OOD score functions on the model output. 
Conversely, MDS \cite{lee2018simple} and KNN \cite{sun2022out} utilize the clustering of ID samples in the feature space and build auxiliary models for OOD detection.
ViM \cite{wang2022vim} combines feature null space information with the output space Energy score. 
In addition, we consider four baseline methods particularly competitive under contrastive loss, $\mathtt{CSI}$~\cite{tack2020csi}, $\mathtt{SSD+}$~\cite{sehwag2020ssd}, and $\mathtt{KNN+}$.
All four methods utilize feature space clustering through building auxiliary models. 
Our method, \fdbd, is training-agnostic and applicable across training schemes. 
We eliminate auxiliary models and incorporate class-specific information from the decision boundaries perspective.
Notably, \fdbd, MSP, and Energy are hyper-parameter free, while the other baselines require hyper-parameter fine-tuning.

\begin{table*}[t]
\footnotesize
\caption{\textbf{fDBD achieves superior performance with negligible latency overhead on ImageNet OOD benchmark.}
Evaluated on ResNet-50 with FPR95, AUROC, and inference latency. 
$\uparrow$ indicates that larger values are better and vice versa.
Best performance highlighted in \textbf{bold}.
Methods with $*$ are hyperparameter-free. 
}
\vspace{-1.8mm}
\label{tab:imagenet}
\begin{center}
\includegraphics[width=0.98\textwidth]{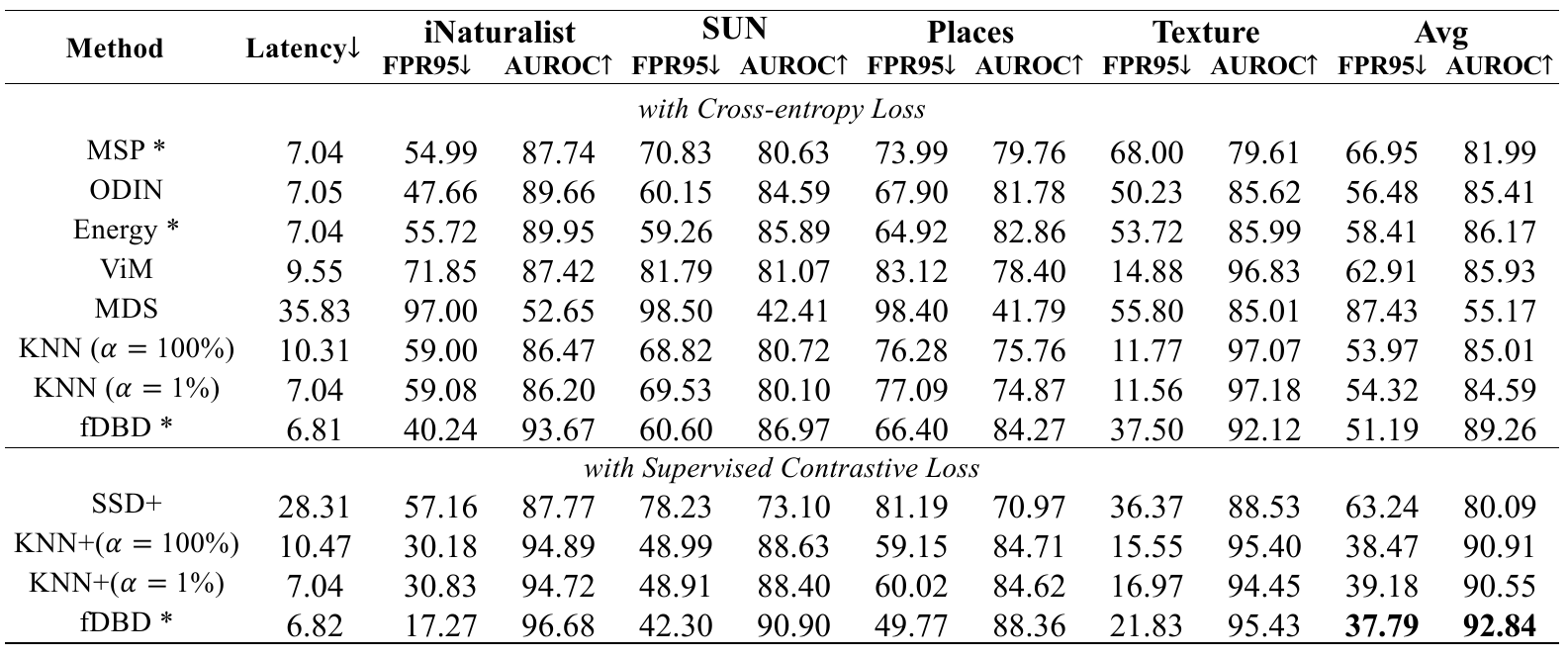}
\end{center}
\vspace{-2mm}
\end{table*}

\textbf{OOD Detection Performance}  
In Table~\ref{tab:cifar10-resnet}, we compare \fdbd~with the baselines. 
Overall, \fdbd~achieves state-of-art performance in terms of FPR95 and AUROC scores across training schemes. 
In addition, thanks to our efficient distance estimation method in Section~\ref{sec:dist-measure}, \fdbd~has minimal computational overhead:
the original classifier takes 0.53 milliseconds per image, and with \fdbd, the processing time remains the same. 
Furthermore, we observe that OOD detection significantly improves under contrastive learning. 
This aligns with the study by \citet{sun2022out}, showing that contrastive learning better separates ID and OOD features.

We highlight three groups of comparisons: 
\begin{itemize}[left=0pt]
    \item \textbf{\fdbd~v.s. MSP / Energy}: 
    All three methods are hyperparameter-free and detect OOD based on model uncertainty:
    MSP and Energy use softmax confidence and Energy score in the output space, respectively, whereas \fdbd~utilizes the feature-space distance \emph{w.r.t.} decision boundaries.
    Looking into the performance in Table~\ref{tab:cifar10-resnet}, on the model trained with cross-entropy loss, our \fdbd~reduces the average FPR95 of MSP by 29.69\%, which is a relatively \textbf{48.85\%} reduction in error. 
    Additionally, \fdbd~reduces the average FPR95 of Energy by 13.78\%,  resulting in a relatively \textbf{30.73\%} reduction in error.
    The substantial improvement aligns with our intuition that the feature space contains crucial information for OOD detection, which we leverage in both our uncertainty metric and our regularization scheme.
    \item \textbf{\fdbd~v.s. KNN}
    We benchmark against KNN under the same hyperparameter setup in \citet{sun2022out}, using $k = 50$ nearest neighbors across the entire training set. 
    While both \fdbd~and KNN achieve superior detection effectiveness on CIFAR-10 OOD benchmark, KNN reports an average inference time of $1.93ms$ per image, inducing a noticeable overhead in comparison to \fdbd~due to the use of the auxiliary model. 
    In addition, \fdbd~significantly outperforms KNN on ImageNet OOD benchmark in Table~\ref{tab:imagenet}, highlighting the benefit of incorporating the class-specific information from the class decision boundary perspective. 
    \item \textbf{\fdbd~v.s.ViM} 
    \fdbd~and ViM \cite{wang2022vim} both integrate class-specific information into feature space representation.
    Specifically, ViM algebraically adds the output space energy score to the feature null space score. 
    Due to the use of null space, ViM requires expensive matrix multiplication during inference, resulting in a noticeable latency increase of $0.70ms$ compared to \fdbd.
    Moreover, \fdbd~outperforms ViM, especially on ImageNet OOD benchmark in Table~\ref{tab:imagenet}.
    This suggests the effectiveness of our geometrically motivated integration of class-specific information from the perspective of feature-space class decision boundaries, compared to simplly algebraically adding output-space scores to feature-space scores, as done in ViM.
\vspace{-2mm}
\end{itemize}


\vspace{-2mm}

\begin{table*}
\vspace{-3mm}
\caption{\textbf{fDBD achieves competitive performance on ViT-B/16 model fine-tuned on ImageNet-1k.} 
Evaluated under AUROC. 
Best performance highlighted in \textbf{bold}.
}
\vspace{-1mm}
\label{tab:vit}
\begin{center}
\vspace{-1mm}
\includegraphics[width=0.75\textwidth]{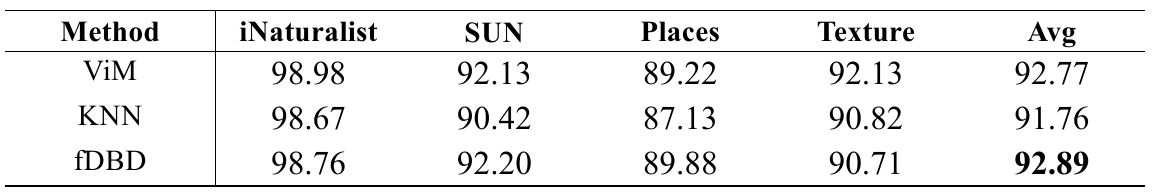}
\end{center}
\vspace{-4mm}
\end{table*}
\begin{table*}[t]
\caption{fDBD is compatible with activation shaping algorithms ReAct, ASH, and Scale.
Evaluated under AUROC and FPR95 on ImageNet OOD Benchmark. 
Best performance highlighted in \textbf{bold}.
}
\label{tab:w_react}
\footnotesize
\centering
\includegraphics[width=0.94\textwidth]{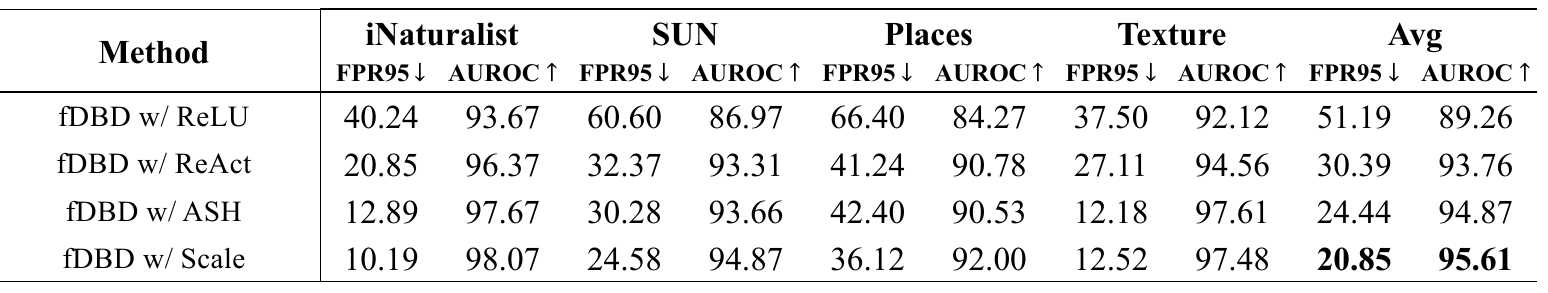}
\vspace{-1mm}
\end{table*}
\subsection{Evaluation on ImageNet Benchmarks}\label{sec:imagenet}

In Table~\ref{tab:imagenet}
\footnote{
Results in Table~\ref{tab:imagenet} except ours and ViM are from Table 4 by \citet{sun2022out}.
MDS here refers to Mahalanobis there. 
Following the reference table, we exclude CSI, since \citet{sun2022out} note that training of CSI on ImageNet is notably resource-intensive, requiring three months on 8 Nvidia 2080Tis. 
}, we further compare the efficiency and effectiveness of our \fdbd~and baselines on larger scale ImageNet OOD Benchmarks on ResNet-50. 

\noindent\textbf{Training Schemes $\&$ Datasets $\&$ Baselines}
We consider the training schemes discussed in Section~\ref{sec:cifar} and examine models trained with cross-entropy loss and supervised contrastive loss. 
The ResNet-50 trained under cross-entropy loss achieves an accuracy of 76.65\% and the ResNet-50 trained under supervised contrastive loss achieves an accuracy of 77.30\%.

We use 50,000 ImageNet validation images in the standard split as ID test samples. Following~\citet{huang2021mos} and \citet{sun2022out},
we remove classes in Texture, Places365~\cite{zhou2017places}, iNaturalist~\cite{van2018inaturalist}, SUN~\cite{xiao2010sun} that overlap with ImageNet and use the remaining datasets as OOD samples. 
All images are of size $224 \times 224$.

We compare to the same baselines in Section~\ref{sec:cifar} except for CSI.
For KNN, we consider two sets of hyper-parameters reported in the original paper \cite{sun2022out}: $\alpha = 100\%$ refers to searching through all training data for $k = 1000$ nearest neighbors;  $\alpha = 1\%$ refers to searching through sampled $1\%$ of training data for $10$ nearest neighbors. 

\noindent\textbf{OOD Detection Performance}
Table~\ref{tab:imagenet} shows that \fdbd~outperforms all baselines in both average FPR95 and average AUROC on ImageNet OOD benchmarks.
This demonstrates \fdbd~consistently maintains its superior effectiveness in OOD detection on large-scale datasets. 
In addition, \fdbd~remains computationally efficient for ImageNet OOD detection. 
This aligns with our observation on CIFAR-10 benchmarks and supports our analysis that \fdbd~scales linearly with the class number and the dimension, ensuring manageable computation for large models and datasets. 

\subsection{Evaluation on Alternative Architectures}

To examine the generalizability of our proposed method beyond ResNet, we further experiment with transformer-based ViT model \cite{dosovitskiy2020image} and DenseNet~\citep{huang2017densely}. 
In Table~\ref{tab:vit}, we evaluate our~\fdbd, as well as strong competitors ViM and KNN on a ViT-B/16 fine-tuned with ImageNet-1k using cross-entropy loss.
The classifier achieves an accuracy of 81.14\%. 
We consider the same OOD test sets as in Section~\ref{sec:imagenet} for Imagenet. 
In Appendix~\ref{sec:densenet}, we extend our experiments to DenseNet. 
The performance on ViT and DenseNet demonstrates the effectiveness of \fdbd~across different network architectures.


\subsection{Evaluation under Activation Shaping}
Orthogonal to the effort of designing standalone detection scores, \citet{sun2021react, djurisic2022extremely} and \citet{xu2023scaling} propose to shape the feature activation to improve ID/OOD separation. 
The proposed algorithms, ReAct \citep{sun2021react}, ASH \citep{djurisic2022extremely}, and Scale \citep{xu2023scaling}, serve as alternative operations to the standard ReLU activation in our experiments so far.
With proper hyper-parameter selection, such algorithms have been shown to enhance the performance of standalone scores such as Energy, as detailed in Appendix~\ref{app:baseline}. 
As a hyperparameter-free method, our \fdbd~can be seamlessly combined with hyperparameter-dependent activation shaping algorithms without intricate tuning interactions.
In Table~\ref{tab:w_react}, we compare \fdbd~performance under standard ReLU activation and under activation shaping algorithms ReAct, ASH, and Scale. 
Specifically, we evaluate ImageNet OOD Benchmarks on a ResNet-50 trained under cross-entropy loss following detailed setups in Section~\ref{sec:imagenet}. 
For hyperparameter selection, we adhere to the original papers and set the percentile values to 80, 90, 90 for ReAct, ASH, and Scale, respectively. 
With activation shaping applied both to test features and the mean of training feature in Equation~\ref{eq:regDBScore}, we observe improved performance across OOD datasets, validating the compatibility of \fdbd~with ReAct, ASH, and Scale. 
We remark that \fdbd~with Scale achieves the state-of-art performance on this benchmark, comparable to Energy with Scale, as detailed in  Appendix~\ref{sec:react_full}.


\subsection{Ablation Study}\label{sec:ablation}

\subsubsection{Effect of Regularization}

Previously, we illustrate in Figure~\ref{fig:regularization} that regularization enhances the ID/OOD separation under the metric of feature distances to decision boundaries.
We now quantitatively study the regularization effect.
Specifically, we compare the performance of OOD detection using the regularized average distance $\mathtt{regDistDB}$, the regularization term $\|z - \mu_{train} \|_2$, and the un-regularized average distance 
$$ \mathtt{avgDistDB} \coloneqq \|z - \mu_{train} \|_2 \ \mathtt{regDistDB} $$
as detection scores respectively.
Experiments are conducted on a ResNet-50 trained under cross-entropy loss following detailed setups in Section~\ref{sec:imagenet}.
We report the performance in AUROC scores in Table~\ref{tab:ablation-reg} and FPR95 in Appendix~\ref{app:fpr95_ablation}.
Aligning with Figure~\ref{fig:intuition}, $\|z - \mu_{train} \|_2$ alone does not necessarily distinguish between ID and OOD samples, as indicated by AUROC scores around 50. 
However, regularization with respect to $\|z - \mu_{train} \|_2$ enhances ID/OOD separation. 
Consequently, $\mathtt{regDistDB}$ improves over $\mathtt{avgDistDB}$ and achieves higher AUROC, as shown in Table~\ref{tab:ablation-reg}.   
This supports our intuition in Section~\ref{sec:method} to compare ID/OOD at equal deviation levels through regulirization. 
We further theoretically explain the observed enhancement in Appendix~\ref{app: regularization}. 

\begin{table}
\vspace{-2mm}
\caption{\textbf{Regularization enhances the effectiveness of OOD detection.} 
AUROC scores reported on ImageNet Benchmarks (higher is better). $\mathtt{regDistDB}$ outperforms $\mathtt{avgDistDB}$.
}
\label{tab:ablation-reg}
\footnotesize
\centering
\includegraphics[width=\columnwidth]{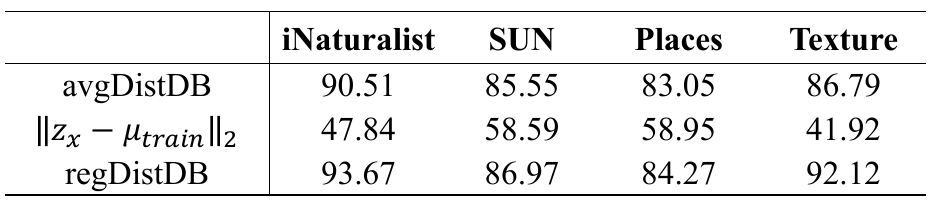}
\vspace{-2mm}
\end{table}

\subsubsection{Effect of Individual Distances}
For \fdbd, we design the detection score as the feature distances to the decision boundaries, averaged over \emph{all} unpredicted classes. 
Notably, \fdbd~operates as a hyperparameter-free method, and we do \textbf{not} tune the number of distances in our experiments. 
Nevertheless, we perform an ablation study to understand the effect of individual distances. 

To align across samples predicted as different classes, we sort per sample the feature distances to decision boundaries.
We then detect OOD using the average of top-$k$ smallest distance values. 
Specifically, $k=1$ corresponds to the detection score being the ratio between the feature distance to the closest decision boundary and the feature distance to the mean of training features. 
And $k=9$ on CIFAR-10 and $k=999$ on ImageNet recover our detection score $\mathtt{regDistDB}$ (see Eqn.~\eqref{eq:regDBScore}), where we average over all distances for OOD detection. 

We experiment with CIFAR-10 and ImageNet benchmarks on ResNets trained with cross-entropy loss, following the setups in Section~\ref{sec:cifar} and Section~\ref{sec:imagenet}. 
In Figure~\ref{fig:ablation_k}, we present the average FPR95 and AUROC score across OOD datasets, using $k$ distances for detection. 
Looking into Figure~\ref{fig:ablation_k}, the performance improves with increasing number of $k$.
This justifies our design of \fdbd~as a hyper-parameter-free method, utilizing all distances for OOD detection.  

\begin{figure}
\begin{center}
\includegraphics[width=0.95\columnwidth]{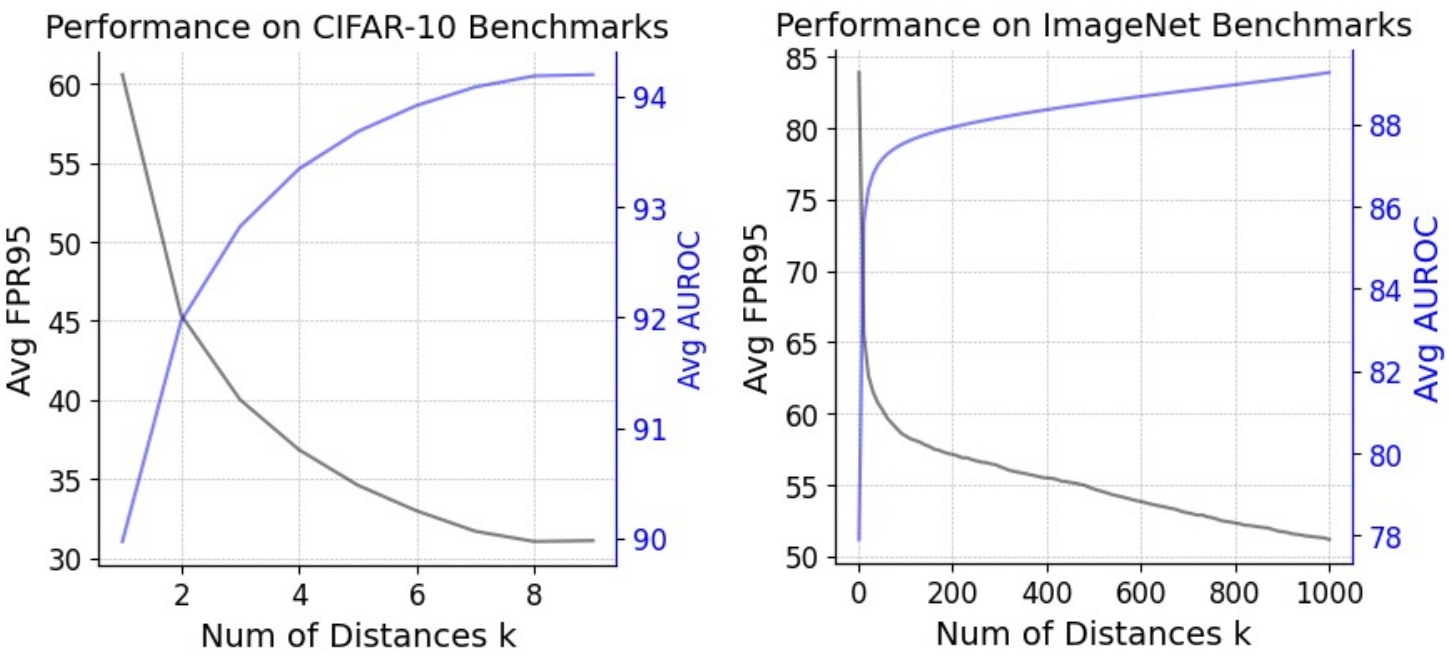}
\end{center} 
\vspace{-4mm}
\caption{
\textbf{Ablation on Individual Distances.} 
\textit{Left:} CIFAR-10 Benchmark performance improves with an increasing number of distances.
\textit{Right:} ImageNet Benchmark performance improves with an increasing number of distances.
The performance supports the use of all distances in our hyperparameter-free \fdbd.
}
\label{fig:ablation_k}
\end{figure}

\section{Theoretical Analysis}\label{sec:justification}

In this section, we give theoretical analysis to shed light on our observation and algorithm design in Section~\ref{sec:method}.

\vspace{-3mm}
\paragraph{Setups}
We consider a general classifier for a class set $\mathcal{C}$ with a penultimate layer of dimension $P$.
Following \citet{lee2018simple}, we model the ID feature distribution as a Gaussian mixture.
Specifically, we consider $|\mathcal{C}|$ equally-weighted components, where each component corresponds to a class $i \in \mathcal{C}$ and follows a Gaussian distribution $N(\bm{\mu}_i, \sigma^2 \bm{I})$, where $\bm{I}$ is the identity matrix. 
Without loss of generality, we assume the distribution is zero-centered, i.e. $\bm{\mu} \doteq \frac{1}{|\mathcal{C}|} \sum_{i \in \mathcal{C}} \bm{\mu}_i = \bm{0}$.
Following the empirical observation by \citet{papyan2020prevalence}, we model the geometry of class means $\{ \mu_i \}$ as a simplex Equiangular Tight Framework (ETF): 
$$  \| \bm{\mu}_i \|_2 =  \| \bm{\mu}_j \|_2 \ \ \forall i, j, $$ 
$$\< \frac{\bm{\mu}_i}{\|\bm{\mu}_i\|_2 },  \frac{\bm{\mu}_j }{\|\bm{\mu}_j \|_2 } \>   =  \frac{|\mathcal{C}|}{|\mathcal{C}|-1} \delta_{i,j} - \frac{1}{|\mathcal{C}| -1},$$
where $\delta_{ij}$ is the Kronecker delta symbol.

Under the modeling, the optimal decision region of class $i$ can be defined as:
$$ \mathcal{V}_i \doteq \{ \bm{z}: \<\bm{\mu}_i,\bm{z}\> \geq \max_{j\neq i}\<\bm{\mu}_j, \bm{z}\>\}.$$
Correspondingly, the decision boundary between class $i$ and $j$ is:
$$ \mathcal{S}_{ij} \doteq \{\bm{z}: \<\bm{\mu}_i,\bm{z}\>=\<\bm{\mu}_j,\bm{z}\>\geq \max_{k\neq i,j}\<\bm{\mu}_k,\bm{z}\> \}. $$
For any $\bm{z} \in \mathcal{V}_i$, the distance from $\bm{z}$ to the decision boundary between class $i$ and $j$, $\mathcal{S}_{ij}$, is the length of the projection of $\bm{z}$ onto the norm vector of $\mathcal{S}_{ij}$: 
$$ d(\bm{z}, \mathcal{S}_{ij}) \doteq \frac{ \< \bm{z}, \bm{\mu}_i - \bm{\mu}_j \>}{\| \bm{\mu}_i - \bm{\mu}_j \|}. $$
For simplicity of notation, we denote the union of decision boundaries as $\mathcal{S} = \cup \mathcal{S}_{ij}$. 
Additionally, we denote the distance from $\bm{z}$ to its closest decision boundary as $d(\bm{z}, \mathcal{S})$.

Following \citet{sun2022out}, we assume that OOD features reside outside the dense region of ID feature distribution. 
We define this dense region as the area within two standard deviations from each class mean: $$\mathcal{I} \doteq \cup_i \mathcal{I}_i = \cup_i \{\bm{z}: \|\bm{z} - \bm{\mu}_i\| \leq 2\sigma\}. $$ 
We also assume that ID features are well-separated, so that the dense region of each class is entirely within its decision region, i.e., $\mathcal{I}_i \subset \mathcal{V}_i$. 

\paragraph{Main Result}
Recall from Section~\ref{sec:method} that we observe the feature distance to decision boundaries increases as features deviate from the mean of training features $\bm{\mu}_{train}$.
This observation motivates us to compare ID and OOD features at equal deviation levels and design our detection algorithm accordingly. 
Note that $\bm{\mu}_{train}$ is an empirical estimation of $\bm{\mu}$, the mean of ID feature distribution.

To further understand our observation, we present Proposition~\ref{prop:prop1}, which demonstrates that the feature distance to the decision boundary $d(\mathcal{S}, z)$ increases as $z$ deviates from $\bm{\mu}$.
Additionally, we validate our detection algorithm in Proposition~\ref{prop:prop2}, showing that, at equal deviation levels, ID features tend to be further from the decision boundary compared to OOD features.
We present the complete proofs in Appendix~\ref{app:proof}. 

As discussed in Setups, we assume without loss of generality that the features are zero-centered, i.e., $\bm{\mu} = \bm{0}$. 

\begin{proposition}\label{prop:prop1}
Consider the set of features of equal distance $r$ to the ID distribution mean $\mathcal{E}_r \doteq \{\bm{z}: \|\bm{z} - \bm{\mu}\| = \|\bm{z} \| = r \}$.
For any $r_0 < r_1$, we have: 

\vspace{-2mm}
\begin{equation}
{
\begin{aligned}
    & \frac{1}{Vol(\mathcal{E}_{r_0})}\int_{z \in \mathcal{E}_{r_0}} d(z, \mathcal{S}) \, \mathrm{d(z)} \\
    &  <  \frac{1}{Vol(\mathcal{E}_{r_1})}\int_{z \in \mathcal{E}_{r_1}} d(z, \mathcal{S}) \, \mathrm{d(z)}.  
\end{aligned} }
\end{equation}
\end{proposition}

\vspace{0.5mm}
\begin{proposition}\label{prop:prop2}
Consider ID and OOD features of equal distance $r$ to the ID distribution mean, where $\sigma < r < 5\sigma$.
For ID region, $\mathcal{I} \cap \mathcal{E}_r$, and OOD region, $\mathcal{I}^\complement \cap \mathcal{E}_r$, we have 

\vspace{-4mm}
\begin{equation}
{
\begin{aligned}
    & \frac{1}{Vol(\mathcal{I} \cap \mathcal{E}_r)}\int_{z \in \mathcal{I} \cap \mathcal{E}_r} d(z, \mathcal{S}) \, \mathrm{d(z)} \\
    &  > \frac{1}{Vol(\mathcal{I}^\complement \cap \mathcal{E}_r)} \int_{z \in \mathcal{I}^\complement \cap \mathcal{E}_r} d(z, \mathcal{S}) \, \mathrm{d(z)}.
\end{aligned}}
\end{equation}
\end{proposition}

\section{Related Work}

An extensive body of research work has been focused on developing OOD detection algorithms.
And we refer readers to comprehensive literature reviews by \citet{yang2021oodsurvey, yang2022openood, yang2022fsood, zhang2023openood, bitterwolf2023ninco}. 
Particularly, one line of work is post-hoc and builds upon pre-trained models. 
For example, \citet{liang2018enhancing, hendrycks2019scaling} and \citet{liu2020energy} design OOD score over the output space of a classifier, 
whereas \citet{lee2018simple, sun2022out, ndiour2020out} and \citet{liu2023detecting} measure OOD-ness using \emph{feature} space information. 
Moreover, \citet{huang2021importance} explore OOD detection from the gradient space. 
Our work builds on the \emph{feature} space and investigates from the largely under-explored perspective of decision boundaries.
Orthogonality, \citet{sun2021react, djurisic2022extremely} and \citet{xu2023scaling} reveal that activation shaping on pre-trained models can enhance the ID/OOD separation and improves the performance of standalone detection scores in general. 
Our experiments validates that \fdbd~ is also compatible with activation shaping methods.

Another line of work explores the regularization of OOD detection in training.
For example, \citet{devries2018learning} and \citet{hsu2020generalized} propose OOD-specific architecture whereas \citet{wei2022mitigating, huang2021mos} and \citet{ming2022exploit} design OOD-specific training loss. 
In addition, \citet{tack2020csi} propose an OOD-specific contrastive learning scheme, while \citet{tao2023non} and \citet{du2022vos} explore methods for constructing virtual OOD samples to facilitate OOD-aware training.
Recently, \citet{fort2021exploring} reveal that finetuning a visual transformer with OOD exposure significantly can improve OOD detection performance. 
Our work does not assume specific training schemes and does not belong to this school of work. 

\section{Conclusion}

In this work, we propose an efficient and effective OOD detector \fdbd~ based on the novel perspective of feature distances to decision boundaries. We first introduce a closed-form estimation to measure the feature distance to decision boundaries. 
Based on our estimation method, we reveal that ID samples tend to reside further away from the decision boundary than OOD samples. 
Moreover, we find that ID and OOD samples are better separated when compared at equal deviation levels from the mean of training features. 
By regularizing feature distances to decision boundaries based on feature deviation from the mean, we design a decision boundary-based OOD detector that achieves state-of-the-art effectiveness with minimal latency overhead.
We hope our algorithm can inspire future work to explore model uncertainty from the perspective of decision boundaries, both for OOD detection and other research problems such as adversarial robustness and domain generalization.

\section*{Impact Statement}
This paper presents work whose goal is to advance the field of Machine Learning. There are many potential societal consequences of our work, none of which we feel must be specifically highlighted here.

\bibliography{example_paper}

\begin{thebibliography}{47}
\providecommand{\natexlab}[1]{#1}
\providecommand{\url}[1]{\texttt{#1}}
\expandafter\ifx\csname urlstyle\endcsname\relax
  \providecommand{\doi}[1]{doi: #1}\else
  \providecommand{\doi}{doi: \begingroup \urlstyle{rm}\Url}\fi

\bibitem[Bitterwolf et~al.(2023)Bitterwolf, Mueller, and
  Hein]{bitterwolf2023ninco}
Bitterwolf, J., Mueller, M., and Hein, M.
\newblock In or out? fixing imagenet out-of-distribution detection evaluation.
\newblock In \emph{ICML}, 2023.
\newblock URL \url{https://proceedings.mlr.press/v202/bitterwolf23a.html}.

\bibitem[Carlini \& Wagner(2017)Carlini and Wagner]{carlini2017towards}
Carlini, N. and Wagner, D.
\newblock Towards evaluating the robustness of neural networks.
\newblock In \emph{2017 ieee symposium on security and privacy (sp)}, pp.\
  39--57. Ieee, 2017.

\bibitem[Cimpoi et~al.(2014)Cimpoi, Maji, Kokkinos, Mohamed, and
  Vedaldi]{cimpoi2014describing}
Cimpoi, M., Maji, S., Kokkinos, I., Mohamed, S., and Vedaldi, A.
\newblock Describing textures in the wild.
\newblock In \emph{IEEE Conference in Computer Vision and Pattern Recognition},
  pp.\  3606--3613, 2014.

\bibitem[Deng et~al.(2009)Deng, Dong, Socher, Li, Li, and
  Fei-Fei]{deng2009imagenet}
Deng, J., Dong, W., Socher, R., Li, L.-J., Li, K., and Fei-Fei, L.
\newblock Imagenet: A large-scale hierarchical image database.
\newblock In \emph{2009 IEEE conference on computer vision and pattern
  recognition}, pp.\  248--255. Ieee, 2009.

\bibitem[DeVries \& Taylor(2018)DeVries and Taylor]{devries2018learning}
DeVries, T. and Taylor, G.~W.
\newblock Learning confidence for out-of-distribution detection in neural
  networks.
\newblock \emph{arXiv preprint arXiv:1802.04865}, 2018.

\bibitem[Djurisic et~al.(2022)Djurisic, Bozanic, Ashok, and
  Liu]{djurisic2022extremely}
Djurisic, A., Bozanic, N., Ashok, A., and Liu, R.
\newblock Extremely simple activation shaping for out-of-distribution
  detection.
\newblock \emph{arXiv preprint arXiv:2209.09858}, 2022.

\bibitem[Dosovitskiy et~al.(2020)Dosovitskiy, Beyer, Kolesnikov, Weissenborn,
  Zhai, Unterthiner, Dehghani, Minderer, Heigold, Gelly,
  et~al.]{dosovitskiy2020image}
Dosovitskiy, A., Beyer, L., Kolesnikov, A., Weissenborn, D., Zhai, X.,
  Unterthiner, T., Dehghani, M., Minderer, M., Heigold, G., Gelly, S., et~al.
\newblock An image is worth 16x16 words: Transformers for image recognition at
  scale.
\newblock \emph{arXiv preprint arXiv:2010.11929}, 2020.

\bibitem[Du et~al.(2022)Du, Wang, Cai, and Li]{du2022vos}
Du, X., Wang, Z., Cai, M., and Li, Y.
\newblock Vos: Learning what you don't know by virtual outlier synthesis.
\newblock \emph{arXiv preprint arXiv:2202.01197}, 2022.

\bibitem[Fort et~al.(2021)Fort, Ren, and Lakshminarayanan]{fort2021exploring}
Fort, S., Ren, J., and Lakshminarayanan, B.
\newblock Exploring the limits of out-of-distribution detection.
\newblock \emph{Advances in Neural Information Processing Systems},
  34:\penalty0 7068--7081, 2021.

\bibitem[He et~al.(2016)He, Zhang, Ren, and Sun]{he2016deep}
He, K., Zhang, X., Ren, S., and Sun, J.
\newblock Deep residual learning for image recognition.
\newblock In \emph{Proceedings of the IEEE conference on computer vision and
  pattern recognition}, pp.\  770--778, 2016.

\bibitem[Hendrycks \& Dietterich(2019)Hendrycks and
  Dietterich]{hendrycks2019benchmarking}
Hendrycks, D. and Dietterich, T.
\newblock Benchmarking neural network robustness to common corruptions and
  perturbations.
\newblock \emph{arXiv preprint arXiv:1903.12261}, 2019.

\bibitem[Hendrycks \& Gimpel(2016)Hendrycks and Gimpel]{hendrycks2016baseline}
Hendrycks, D. and Gimpel, K.
\newblock A baseline for detecting misclassified and out-of-distribution
  examples in neural networks.
\newblock \emph{arXiv preprint arXiv:1610.02136}, 2016.

\bibitem[Hendrycks et~al.(2019)Hendrycks, Basart, Mazeika, Mostajabi,
  Steinhardt, and Song]{hendrycks2019scaling}
Hendrycks, D., Basart, S., Mazeika, M., Mostajabi, M., Steinhardt, J., and
  Song, D.
\newblock Scaling out-of-distribution detection for real-world settings.
\newblock \emph{arXiv preprint arXiv:1911.11132}, 2019.

\bibitem[Hsu et~al.(2020)Hsu, Shen, Jin, and Kira]{hsu2020generalized}
Hsu, Y.-C., Shen, Y., Jin, H., and Kira, Z.
\newblock Generalized odin: Detecting out-of-distribution image without
  learning from out-of-distribution data.
\newblock In \emph{Proceedings of the IEEE/CVF Conference on Computer Vision
  and Pattern Recognition}, pp.\  10951--10960, 2020.

\bibitem[Huang et~al.(2017)Huang, Liu, Van Der~Maaten, and
  Weinberger]{huang2017densely}
Huang, G., Liu, Z., Van Der~Maaten, L., and Weinberger, K.~Q.
\newblock Densely connected convolutional networks.
\newblock In \emph{Proceedings of the IEEE conference on computer vision and
  pattern recognition}, pp.\  4700--4708, 2017.

\bibitem[Huang \& Li(2021)Huang and Li]{huang2021mos}
Huang, R. and Li, Y.
\newblock Mos: Towards scaling out-of-distribution detection for large semantic
  space.
\newblock In \emph{Proceedings of the IEEE/CVF Conference on Computer Vision
  and Pattern Recognition}, pp.\  8710--8719, 2021.

\bibitem[Huang et~al.(2021)Huang, Geng, and Li]{huang2021importance}
Huang, R., Geng, A., and Li, Y.
\newblock On the importance of gradients for detecting distributional shifts in
  the wild.
\newblock \emph{Advances in Neural Information Processing Systems},
  34:\penalty0 677--689, 2021.

\bibitem[Khosla et~al.(2020)Khosla, Teterwak, Wang, Sarna, Tian, Isola,
  Maschinot, Liu, and Krishnan]{khosla2020supervised}
Khosla, P., Teterwak, P., Wang, C., Sarna, A., Tian, Y., Isola, P., Maschinot,
  A., Liu, C., and Krishnan, D.
\newblock Supervised contrastive learning.
\newblock \emph{Advances in Neural Information Processing Systems},
  33:\penalty0 18661--18673, 2020.

\bibitem[Krizhevsky et~al.(2009)Krizhevsky, Hinton,
  et~al.]{krizhevsky2009learning}
Krizhevsky, A., Hinton, G., et~al.
\newblock Learning multiple layers of features from tiny images.
\newblock 2009.

\bibitem[Lee et~al.(2018)Lee, Lee, Lee, and Shin]{lee2018simple}
Lee, K., Lee, K., Lee, H., and Shin, J.
\newblock A simple unified framework for detecting out-of-distribution samples
  and adversarial attacks.
\newblock \emph{Advances in neural information processing systems}, 31, 2018.

\bibitem[Liang et~al.(2018)Liang, Li, and Srikant]{liang2018enhancing}
Liang, S., Li, Y., and Srikant, R.
\newblock Enhancing the reliability of out-of-distribution image detection in
  neural networks.
\newblock In \emph{6th International Conference on Learning Representations,
  ICLR 2018}, 2018.

\bibitem[Liu \& Qin(2023)Liu and Qin]{liu2023detecting}
Liu, L. and Qin, Y.
\newblock Detecting out-of-distribution through the lens of neural collapse.
\newblock \emph{arXiv preprint arXiv:2311.01479}, 2023.

\bibitem[Liu et~al.(2020)Liu, Wang, Owens, and Li]{liu2020energy}
Liu, W., Wang, X., Owens, J., and Li, Y.
\newblock Energy-based out-of-distribution detection.
\newblock \emph{Advances in Neural Information Processing Systems},
  33:\penalty0 21464--21475, 2020.

\bibitem[Loshchilov \& Hutter(2016)Loshchilov and Hutter]{loshchilov2016sgdr}
Loshchilov, I. and Hutter, F.
\newblock Sgdr: Stochastic gradient descent with warm restarts.
\newblock \emph{arXiv preprint arXiv:1608.03983}, 2016.

\bibitem[Ming et~al.(2023)Ming, Sun, Dia, and Li]{ming2022exploit}
Ming, Y., Sun, Y., Dia, O., and Li, Y.
\newblock How to exploit hyperspherical embeddings for out-of-distribution
  detection?
\newblock In \emph{The Eleventh International Conference on Learning
  Representations}, 2023.

\bibitem[Ndiour et~al.(2020)Ndiour, Ahuja, and Tickoo]{ndiour2020out}
Ndiour, I., Ahuja, N., and Tickoo, O.
\newblock Out-of-distribution detection with subspace techniques and
  probabilistic modeling of features.
\newblock \emph{arXiv preprint arXiv:2012.04250}, 2020.

\bibitem[Netzer et~al.(2011)Netzer, Wang, Coates, Bissacco, Wu, and
  Ng]{netzer2011reading}
Netzer, Y., Wang, T., Coates, A., Bissacco, A., Wu, B., and Ng, A.~Y.
\newblock Reading digits in natural images with unsupervised feature learning.
\newblock 2011.

\bibitem[Papernot et~al.(2018)Papernot, Faghri, Carlini, Goodfellow, Feinman,
  Kurakin, Xie, Sharma, Brown, Roy, Matyasko, Behzadan, Hambardzumyan, Zhang,
  Juang, Li, Sheatsley, Garg, Uesato, Gierke, Dong, Berthelot, Hendricks,
  Rauber, and Long]{papernot2018cleverhans}
Papernot, N., Faghri, F., Carlini, N., Goodfellow, I., Feinman, R., Kurakin,
  A., Xie, C., Sharma, Y., Brown, T., Roy, A., Matyasko, A., Behzadan, V.,
  Hambardzumyan, K., Zhang, Z., Juang, Y.-L., Li, Z., Sheatsley, R., Garg, A.,
  Uesato, J., Gierke, W., Dong, Y., Berthelot, D., Hendricks, P., Rauber, J.,
  and Long, R.
\newblock Technical report on the cleverhans v2.1.0 adversarial examples
  library.
\newblock \emph{arXiv preprint arXiv:1610.00768}, 2018.

\bibitem[Papyan et~al.(2020)Papyan, Han, and Donoho]{papyan2020prevalence}
Papyan, V., Han, X., and Donoho, D.~L.
\newblock Prevalence of neural collapse during the terminal phase of deep
  learning training.
\newblock \emph{Proceedings of the National Academy of Sciences}, 117\penalty0
  (40):\penalty0 24652--24663, 2020.

\bibitem[Sastry \& Oore(2020)Sastry and Oore]{sastry2020detecting}
Sastry, C.~S. and Oore, S.
\newblock Detecting out-of-distribution examples with gram matrices.
\newblock In \emph{International Conference on Machine Learning}, pp.\
  8491--8501. PMLR, 2020.

\bibitem[Sehwag et~al.(2020)Sehwag, Chiang, and Mittal]{sehwag2020ssd}
Sehwag, V., Chiang, M., and Mittal, P.
\newblock Ssd: A unified framework for self-supervised outlier detection.
\newblock In \emph{International Conference on Learning Representations}, 2020.

\bibitem[Sun et~al.(2021)Sun, Guo, and Li]{sun2021react}
Sun, Y., Guo, C., and Li, Y.
\newblock React: Out-of-distribution detection with rectified activations.
\newblock \emph{Advances in Neural Information Processing Systems},
  34:\penalty0 144--157, 2021.

\bibitem[Sun et~al.(2022)Sun, Ming, Zhu, and Li]{sun2022out}
Sun, Y., Ming, Y., Zhu, X., and Li, Y.
\newblock Out-of-distribution detection with deep nearest neighbors.
\newblock \emph{arXiv preprint arXiv:2204.06507}, 2022.

\bibitem[Tack et~al.(2020)Tack, Mo, Jeong, and Shin]{tack2020csi}
Tack, J., Mo, S., Jeong, J., and Shin, J.
\newblock Csi: Novelty detection via contrastive learning on distributionally
  shifted instances.
\newblock \emph{Advances in neural information processing systems},
  33:\penalty0 11839--11852, 2020.

\bibitem[Tao et~al.(2023)Tao, Du, Zhu, and Li]{tao2023non}
Tao, L., Du, X., Zhu, X., and Li, Y.
\newblock Non-parametric outlier synthesis.
\newblock \emph{arXiv preprint arXiv:2303.02966}, 2023.

\bibitem[Van~Horn et~al.(2018)Van~Horn, Mac~Aodha, Song, Cui, Sun, Shepard,
  Adam, Perona, and Belongie]{van2018inaturalist}
Van~Horn, G., Mac~Aodha, O., Song, Y., Cui, Y., Sun, C., Shepard, A., Adam, H.,
  Perona, P., and Belongie, S.
\newblock The inaturalist species classification and detection dataset.
\newblock In \emph{Proceedings of the IEEE conference on computer vision and
  pattern recognition}, pp.\  8769--8778, 2018.

\bibitem[Wang et~al.(2022)Wang, Li, Feng, and Zhang]{wang2022vim}
Wang, H., Li, Z., Feng, L., and Zhang, W.
\newblock Vim: Out-of-distribution with virtual-logit matching.
\newblock In \emph{Proceedings of the IEEE/CVF conference on computer vision
  and pattern recognition}, pp.\  4921--4930, 2022.

\bibitem[Wei et~al.(2022)Wei, Xie, Cheng, Feng, An, and Li]{wei2022mitigating}
Wei, H., Xie, R., Cheng, H., Feng, L., An, B., and Li, Y.
\newblock Mitigating neural network overconfidence with logit normalization.
\newblock \emph{arXiv preprint arXiv:2205.09310}, 2022.

\bibitem[Xiao et~al.(2010)Xiao, Hays, Ehinger, Oliva, and
  Torralba]{xiao2010sun}
Xiao, J., Hays, J., Ehinger, K.~A., Oliva, A., and Torralba, A.
\newblock Sun database: Large-scale scene recognition from abbey to zoo.
\newblock In \emph{2010 IEEE computer society conference on computer vision and
  pattern recognition}, pp.\  3485--3492. IEEE, 2010.

\bibitem[Xu et~al.(2023)Xu, Chen, Franchi, and Yao]{xu2023scaling}
Xu, K., Chen, R., Franchi, G., and Yao, A.
\newblock Scaling for training time and post-hoc out-of-distribution detection
  enhancement.
\newblock In \emph{The Twelfth International Conference on Learning
  Representations}, 2023.

\bibitem[Xu et~al.(2015)Xu, Ehinger, Zhang, Finkelstein, Kulkarni, and
  Xiao]{xu2015turkergaze}
Xu, P., Ehinger, K.~A., Zhang, Y., Finkelstein, A., Kulkarni, S.~R., and Xiao,
  J.
\newblock Turkergaze: Crowdsourcing saliency with webcam based eye tracking.
\newblock \emph{arXiv preprint arXiv:1504.06755}, 2015.

\bibitem[Yang et~al.(2021{\natexlab{a}})Yang, Zhou, Li, and
  Liu]{yang2021generalized}
Yang, J., Zhou, K., Li, Y., and Liu, Z.
\newblock Generalized out-of-distribution detection: A survey.
\newblock \emph{arXiv preprint arXiv:2110.11334}, 2021{\natexlab{a}}.

\bibitem[Yang et~al.(2021{\natexlab{b}})Yang, Zhou, Li, and
  Liu]{yang2021oodsurvey}
Yang, J., Zhou, K., Li, Y., and Liu, Z.
\newblock Generalized out-of-distribution detection: A survey.
\newblock \emph{arXiv preprint arXiv:2110.11334}, 2021{\natexlab{b}}.

\bibitem[Yang et~al.(2022{\natexlab{a}})Yang, Wang, Zou, Zhou, Ding, Peng,
  Wang, Chen, Li, Sun, Du, Zhou, Zhang, Hendrycks, Li, and
  Liu]{yang2022openood}
Yang, J., Wang, P., Zou, D., Zhou, Z., Ding, K., Peng, W., Wang, H., Chen, G.,
  Li, B., Sun, Y., Du, X., Zhou, K., Zhang, W., Hendrycks, D., Li, Y., and Liu,
  Z.
\newblock Openood: Benchmarking generalized out-of-distribution detection.
\newblock 2022{\natexlab{a}}.

\bibitem[Yang et~al.(2022{\natexlab{b}})Yang, Zhou, and Liu]{yang2022fsood}
Yang, J., Zhou, K., and Liu, Z.
\newblock Full-spectrum out-of-distribution detection.
\newblock \emph{arXiv preprint arXiv:2204.05306}, 2022{\natexlab{b}}.

\bibitem[Zhang et~al.(2023)Zhang, Yang, Wang, Wang, Lin, Zhang, Sun, Du, Zhou,
  Zhang, Li, Liu, Chen, and Li]{zhang2023openood}
Zhang, J., Yang, J., Wang, P., Wang, H., Lin, Y., Zhang, H., Sun, Y., Du, X.,
  Zhou, K., Zhang, W., Li, Y., Liu, Z., Chen, Y., and Li, H.
\newblock Openood v1.5: Enhanced benchmark for out-of-distribution detection.
\newblock \emph{arXiv preprint arXiv:2306.09301}, 2023.

\bibitem[Zhou et~al.(2017)Zhou, Lapedriza, Khosla, Oliva, and
  Torralba]{zhou2017places}
Zhou, B., Lapedriza, A., Khosla, A., Oliva, A., and Torralba, A.
\newblock Places: A 10 million image database for scene recognition.
\newblock \emph{IEEE transactions on pattern analysis and machine
  intelligence}, 40\penalty0 (6):\penalty0 1452--1464, 2017.

\end{thebibliography}
\bibliographystyle{icml2024}

\newpage

\appendix
\onecolumn

\section{Proof for Section~\ref{sec:justification}}\label{app:proof}

Under the setups in Section~\ref{sec:justification}, we slice the geometric space into regions $\mathcal{V}_i^j$, $i, j \in \{1, ..., C\}$, defined by
\vspace{-2mm}
$$\mathcal{V}_i^j \doteq \{\bm{z}: \<\bm{z}, \bm{\mu}_i\> > \<\bm{z}, \bm{\mu}_j\> \ge \max_{k \neq i, j}\<\bm{z}, \bm{\mu}_k\> \}. \vspace{-3mm}$$
Geometrically, $\mathcal{V}_i^j$ represents the region within the decision region $\mathcal{V}_i$ of class $i$ where the second most likely class is $j$. 
For any $\bm{z} \in \mathcal{V}_i^j$, we have $d(\bm{z}, \mathcal{S}) = d(\bm{z}, \mathcal{S}_{ij})$.
In the following, we establish Proposition~\ref{prop:prop1} and Proposition~\ref{prop:prop2} in region $\mathcal{V}_i^j$ for any $i$, $j$, thereby confirming their validity in the entire region thanks to symmetry. 

\textbf{Proof of Proposition~\ref{prop:prop1}}

\begin{proof}
By definition, any $\bm{z}_0 \in \mathcal{E}_{r_0}$ satisfies $ \| \bm{z}_0 \| = r_0$.
Scaling $\bm{z}_0$ by $r_1 / r_0$ yields $\bm{z}_1 = r_1 / r_0 \cdot \bm{z}_0$. 
We have $ \|\bm{z}_1\| = r_1 /r_0 \cdot \| \bm{z}\|_0 = r_1$, indicating that $\bm{z}_1$ is an element of $\mathcal{E}_{r_1}$.
Conversely, for any $\bm{z}_1 \in \mathcal{E}_{r_1}$, we can obtain $\bm{z}_0 = (r_0 / r_1) \cdot \bm{z}_1 \in \mathcal{E}_{r_0}$.
This establishes a one-to-one mapping between elements in $\mathcal{E}_{r_0}$ and $\mathcal{E}_{r_1}$.
Considering any pair $(\bm{z}_0, \bm{z}_1)$, we have
\begin{equation}\label{eq:scaling}
\begin{aligned}
d(\bm{z}_0, \mathcal{S}) = d(\bm{z}_0, \mathcal{S}_{ij}) &
= \frac{\<\bm{z}_0, \bm{\mu}_i - \bm{\mu}_j \>}{\| \bm{\mu}_i - \bm{\mu}_j \|} 
= r_0 / r_1 \cdot \frac{\<\bm{z}_1, \bm{\mu}_i - \bm{\mu}_j\>}{\| \bm{\mu}_i - \bm{\mu}_j \|}
& < \frac{\<\bm{z_1}, \bm{\mu}_i - \bm{\mu}_j\>}{\| \bm{\mu}_i - \bm{\mu}_j \|} = d(\bm{z}_1, \mathcal{S}_{ij}) = d(\bm{z}_1, \mathcal{S}),
\end{aligned}  
\end{equation}
indicating a consistent relative ordering between elements in $\mathcal{E}_{r_0}$ and $\mathcal{E}_{r_1}$ 
Therefore, Proposition~\ref{prop:prop1}, which asserts the ordering of the mean between these two sets, is validated.
\end{proof}

\noindent \textbf{Proof of Proposition~\ref{prop:prop2}}

\begin{proof}
Without loss of generality, we assume that $\| \bm{\mu}_i\| = 1$ for $\forall i \in \mathcal{C}$ and the distance $r = 1$. 
To parameterize the element $\bm{z}$ within region $\mathcal{V}_i^j \cap \mathcal{E}_{r = 1}$ for given $i, j$, we consider the geodesic on sphere $\mathcal{E}_{r = 1}$ that extends from the class mean $\bm{\mu}_i$ to element $\bm{z}$,  and further extends to point $\bm{v} \in \mathcal{S}_{ij} \cap \mathcal{E}_{r = 1}$:
$$ \gamma_v(t) = \cos(t)\bm{\mu}_i + \sin(t)\frac{\bm{v} - \< \bm{v}, \bm{\mu}_i\> \bm{\mu}_i}{\| \bm{v} - \< \bm{v}, \bm{\mu}_i\>\bm{\mu}_i \|}.$$
For any $\bm{z} \in \mathcal{V}_i^j \cap \mathcal{E}_{r = 1}$ and its corresponding $\bm{v}$, we have $\bm{z}$ residing on the geodesic $\gamma_v(t)$ with $t = \arccos{\< \bm{z}, \bm{\mu}_i\>}$. 

Geometrically, along a geodesic $\gamma_v(t)$, the parameter $t$ increases as one moves from the ID region $\mathcal{I} \cap \mathcal{E}_{r = 1}$ to the OOD region $\mathcal{I}^\complement \cap \mathcal{E}_{r = 1}$.
Moreover, $d(\gamma_v(t), \mathcal{S})$ is equivalent to $d(\gamma_v(t), \mathcal{S}_{ij})$ given that the geodesic resides within $\mathcal{V}_i^j$. 
Therefore, to show Proposition~\ref{prop:prop2} holds for $\forall \bm{z} \in \mathcal{V}_i^j \cap \mathcal{E}_{r = 1}$, it suffices to show that the function
$d(\gamma_v(t), \mathcal{S}_{ij})$
decreases with $t$. 
Diving into the derivatives of $d(\gamma_v(t), \mathcal{S}_{ij})$ with respect to $t$, we have: 
\begin{align}
    \| \bm{\mu}_i - \bm{\mu}_j \| \frac{d}{dt} d(\gamma_v(t),\mathcal{S}_{ij}) &=\<\gamma’_v(t),\bm{\mu}_i-\bm{\mu}_j\> = \< - \sin(t)\bm{\mu}_i +\cos(t)\frac{\bm{v} - \< \bm{v}, \bm{\mu}_i\> \bm{\mu}_i}{\| \bm{v} - \< \bm{v}, \bm{\mu}_i\>\bm{\mu}_i \|},\bm{\mu}_i-\bm{\mu}_j\>  \\
    &= -\sin(t) + \frac{\sin(t)}{1-C}  + \frac{\cos(t)}{\| \bm{v} - \< \bm{v}, \bm{\mu}_i\>\bm{\mu}_i \|} \cdot (\<\bm{v}, \bm{\mu}_i\> - \<\bm{v}, \bm{\mu}_j\> - \< \bm{v}, \bm{\mu}_i \>+ \frac{\<\bm{v}, \bm{\mu}_i \>}{1-C})\\
    &= \frac{C}{1-C} (\sin(t) + \frac{1}{\| \bm{v} - \< \bm{v}, \bm{\mu}_i\>\bm{\mu}_i \|}\cos(t)\< \bm{v}, \bm{\mu}_j \> ). \label{eq:derivitive}
\end{align}
We remark that Eqn.~\ref{eq:derivitive} remains negative within the feasible range of parameter $t$, where $\sin(t) > 0$ and $\cos(t) > 0$. 
This is because the parameter $t$ has its minimum at $\bm{\mu}_i$ with $t_{min} = 0$ and reaches max at $\bm{v}$ with $ t_{max} = \T{\arccos}(\<\bm{v},\bm{\mu}_i\>)$.
As $\<\bm{v},\bm{\mu}_i \> > 0$ from the definition of $\mathcal{V}_i^j$, we have $t_{max} < \frac{\pi}{2}$, ensuring that $t$ remains within the interval $t \in (0, \frac{\pi}{2})$.
\end{proof}

\section{Theoretical Justification for Performance Enhancement through Regularization}\label{app: regularization}

In the following, \( x \) denotes the feature distance to the training feature mean, and \( y \) denotes the feature distance to decision boundaries.
$f_{xy}$ and $g_{xy}$ denote the joint probability density functions of  \( x \) and  \( y \) for ID and OOD samples, respectively.
The notation in this section may vary from the rest of the paper for clarity and ease of presentation within this context. Please refer to the corresponding sections for consistent notation throughout the paper.

In Section~\ref{sec:fdbd}, we regularize \( y \) with respect to \( x \) to compare the distance of ID/OOD features to decision boundaries at the same deviation levels from the training feature mean. 
Eqn.~\ref{eq:scaling} provides intuition on how our regularization effectively enables comparison at the same deviation level \( x \), as \( y \)  scales linearly with \( x \) under our modeling. 
Thus, the regularization effectively conditions \( y \) on \( x \). 

In Proposition~\ref{prop:regularization} below, we analytically justify why conditioning enhances ID/OOD separation, thereby explaining the regularization-induced enhancement observed in Section~\ref{sec:fdbd} and Section~\ref{sec:ablation}.
Specifically, as Figure~\ref{fig:intuition} (Section~\ref{sec:fdbd}) and Table~\ref{tab:ablation-reg} (Section~\ref{sec:ablation}) show the ID and OOD samples cannot be distinguished by \( x \) alone, we consider the case where the marginal distribution of \( x \) is the same for ID and OOD, i.e., $f_x = g_x$. 

\begin{proposition}\label{prop:regularization}
Under Kullback–Leibler (KL) divergence $D_{KL}$, we have:   
$$D_{KL} (f_y || g_y) \leq D_{KL}(f_{y|x} || g_{y|x}).$$
Here, $f_y$ and $g_y$ denote the marginal distribution of feature distance to decision boundaries for ID and OOD samples respectively, whereas $f_{y|x}$ and $g_{y|x}$ denote the conditional distribution \emph{w.r.t.} feature deviation level from the training feature mean for ID and OOD samples respectively. 
\end{proposition}

\begin{proof}
Following the chain rule of KL divergence, we have
$$ D_{KL} (f_{xy} || g_{xy}) = D_{KL} (f_x || g_x) + D_{KL} (f_{y|x} || g_{y|x}). $$
Symmetrically, we also have: 
$$ D_{KL} (f_{xy} || g_{xy}) = D_{KL} (f_y || g_y) + D_{KL} (f_{x|y} || g_{x|y}). $$
Combining both, we have: 
$$ D_{KL} (f_{y} || g_{y}) = D_{KL} (f_x || g_x) + D_{KL} (f_{y|x} || g_{y|x}) - D_{KL} (f_{x|y} || g_{x|y}).$$
Remind that $D_{KL} (f_x || g_x) = 0$, as $f_x = g_x$. 
Also, $D_{KL} (f_{x|y} || g_{x|y}) \geq 0$ due to the non-negativity of KL divergence. 
Therefore, we have: 
$$D_{KL} (f_y || g_y) \leq D_{KL}(f_{y|x} || g_{y|x}).$$
\end{proof}

\section{Evaluation on DenseNet}\label{sec:densenet}

We now extend our evaluation to DenseNet~\citep{huang2017densely}.
The CIFAR-10 classifier we evaluated with achieves a classification accuracy of $94.53\%$. 
We consider the same OOD test sets as in Section~\ref{sec:cifar}. 
The performance shown in Table~\ref{tab:densenet} further indicates the effectiveness and efficiency of our proposed detector across different network architectures.

\begin{table*}
\caption{\textbf{fDBD achieves superior performance with negligible latency overhead on DenseNet.} 
Evaluated with FPR95, AUROC, and inference latency. 
$\uparrow$ indicates that larger values are better and vice versa.
Best performance highlighted in \textbf{bold}.
}
\label{tab:densenet}
\begin{center}
\includegraphics[width=0.94\textwidth]{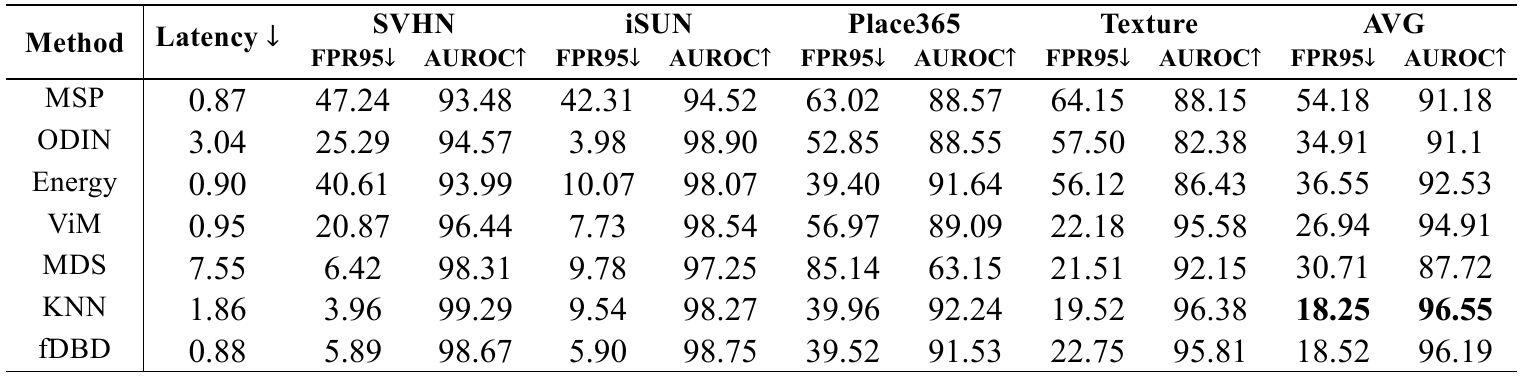}
\end{center}
\vspace{-3mm}
\end{table*}

\section{Evaluation under Activation Shaping}\label{sec:react_full}

In Table~\ref{tab:w_react_full}, we compare the performance of \fdbd~and Energy under activation shaping methods ReAct, ASH, and Scale. 
For both \fdbd~and Energy, we follow the original paper and set the value of the percentile hyperparameter to 80, 90, 90 for ReAct, ASH, and Scale, respectively. 
Experiments are on an ImageNet ResNet-50 classifiers following the detailed setups in Section~\ref{sec:imagenet}. 
Looking into Table~\ref{tab:w_react_full}, we observe that \fdbd~with Scale achieves state-of-art performance on this benchmark, comparable to Energy with Scale. 
\begin{table*}
\footnotesize
\centering
\caption{fDBD is competitive compared to Energy under activation shaping algorithms ReAct, ASH, and Scale on ImageNet Benchmark. 
 }\label{tab:w_react_full}
\vspace{1mm}
\includegraphics[width=0.89\textwidth]{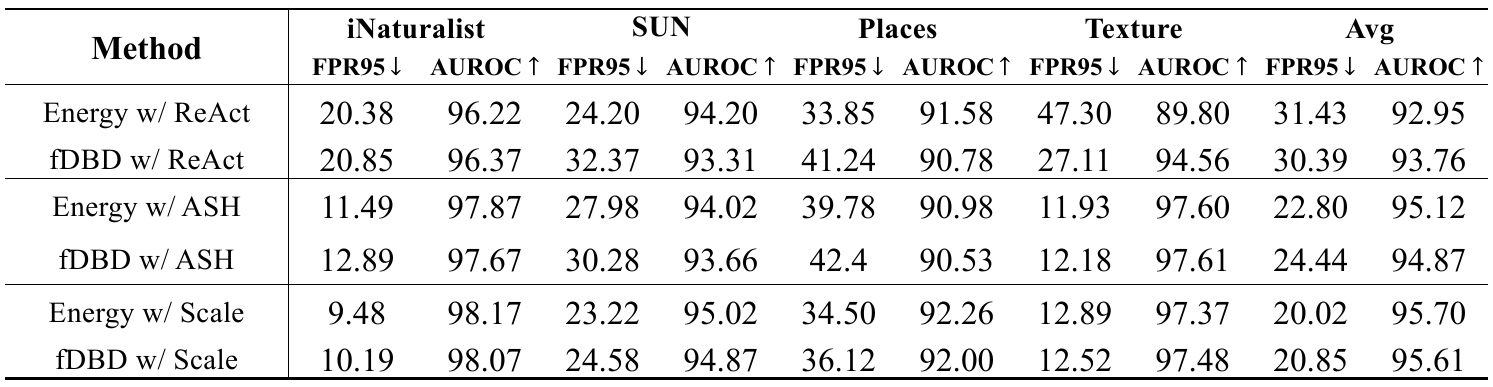}
\vspace{-3mm}
\end{table*}

\section{Evaluation under Domain Shift}

Our \fdbd, as a detector for semantic shift induced by mismatch in training/test class types,  remains effective when ID samples undergo moderate domain shift in real life. 
In Table~\ref{tab:domain_shift}, we compare fDBD performance with clean and moderately corrupted ID samples on CIFAR-10 benchmarks. 
Specifically, we consider CIFAR-10-C \cite{hendrycks2019benchmarking} with severity level 1 \& 2. 
For each severity level, we construct an aggregated dataset by sampling in total 10,000 images from all 4 classes of corruption: Noise, Blur, Weather, and Digital. 
For the rest of the setups, we follow Section~\ref{sec:cifar} and report the average AUROC across OOD datasets. 
As shown in Table~\ref{tab:domain_shift}, fDBD’s performance degrades slightly as the corruption level increases.
Nevertheless, fDBD remains highly effective within a moderate range of domain shift. 
\vspace{-3mm}

\begin{table*}[h]
\centering
\caption{Performance of fDBD with CIFAR-10 / CIFAR-10-C as ID samples on CIFAR-10 Benchmark. 
 }\label{tab:domain_shift}
\vspace{1mm}
\includegraphics[width=0.5\textwidth]{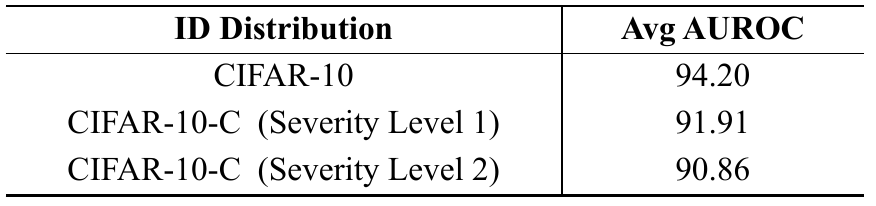}
\vspace{-1mm}
\end{table*}

\section{Implementation Details}\label{app:imp-detail}

\subsection{CIFAR-10}

\textbf{ResNet-18 w/ Cross Entropy Loss}
For experiments presented in Figure~\ref{fig:overview}~\textit{Right}, Figure~\ref{fig:regularization}, Figure~\ref{fig:intuition}, Figure~\ref{fig:ablation_k}~\textit{Left}, Table~\ref{tab:domain_shift} and part of Table~\ref{tab:cifar10-resnet}, we evaluate on a CIFAR-10 classifier of ResNet-18 backbone trained with cross entropy loss. 
The classifier is trained for 100 epochs, with a start learning rate 0.1 decaying to 0.01, 0.001, and 0.0001 at epochs 50, 75, and 90 respectively. 

\noindent \textbf{ResNet-18 w/ Contrastive Loss}
For part of Table~\ref{tab:cifar10-resnet}, we experiment with a CIFAR-10 classifier of the ResNet-18 backbone trained with supcon loss. 
Following \citet{khosla2020supervised}, the model is trained with for 500 epochs with batch size 1024. 
The temperature is set to 0.1.
The cosine learning rate \cite{loshchilov2016sgdr} starts at 0.5 is used. 

\noindent \textbf{DenseNet-101 w/ Cross Entropy Loss}
For experiments presented in Table~\ref{tab:densenet}, we evaluate on a CIFAR-10 classifier of DenseNet-101 backbone. 
The classifier is trained following the set up in \cite{huang2017densely} with depth $L = 100$ and growth rate $k = 12$.

\subsection{ImageNet}

\textbf{ResNet-50 w/ Cross-Entropy Loss}
For evaluation on ImageNet in Figure~\ref{fig:ablation_k}~\textit{Right},  part of Table~\ref{tab:imagenet}, Table~\ref{tab:w_react}, 
Table~\ref{tab:ablation-reg}, 
and Table~\ref{tab:w_react_full} we use the default ResNet-50 model trained with cross-entropy loss provided by Pytorch.
See training recipe here: 
\url{https://pytorch.org/blog/how-to-train-state-of-the-art-models-using-torchvision-latest-primitives/}.

\textbf{ResNet-50 w/ Supervised Contrastive Loss}
For part of Table~\ref{tab:imagenet}, we experiment with a ImageNet classifier of ResNet-50 backbone trained with supcon loss. 
Following \citet{khosla2020supervised}, the model is trained with for 700 epochs with batch size 1024. 
The temperature is set to 0.1.
The cosine learning rate \cite{loshchilov2016sgdr} starts at 0.5 is used. 

\textbf{ImageNet ViT}
In Table~\ref{tab:vit}, we evaluate on the pytorch implementation of ViT and the default checkpoint, available \url{https://github.com/lukemelas/PyTorch-Pretrained-ViT/tree/master}.

\section{Baseline Methods}\label{app:baseline}
We provide an overview of our baseline methods in this session. 
We follow our notation in Section~\ref{sec:method}.
In the following, a lower detection score indicates OOD-ness.

\textbf{MSP} 
\citet{hendrycks2016baseline} propose to detect OOD based on the maximum softmax probability. 
Given a test sample $\bm{x}$, the detection score of MSP can be represented as: 

\begin{equation}
\frac{\exp{(\bm{w}_{f(\bm{x})}^T \bm{z_x}} + b_{f(\bm{x})})}{\sum_{c \in \mathcal{C}} \exp{(\bm{w}_c^T \bm{z_x}} + b_c)},
\end{equation}

\noindent where $\bm{z_x}$ is the penultimate feature space embedding of $\bm{x}$.
Note that calculating the denominator of the softmax score function is an $\Omega(|\mathcal{C}|T(\texttt{exp}))$ operation, where $T(\texttt{exp})$ is the computational complexity for evaluating the exponential function, which is precision related and non-constant. 
Note that the on-device implementation of exponential functions often requires huge look-up tables, incurring significant delay and storage overhead. 
Overall, the computational complexity of MSP on top of the inference process is $\Omega(|\mathcal{C}|T(\texttt{exp}))$. 


\textbf{ODIN}
\citet{liang2018enhancing} propose to amplify ID and OOD separation on top of MSP through temperature scaling and adversarial perturbation. 
Given a sample $\bm{x}$, ODIN constructs a noisy sample $\bm{x'}$ from $\bm{x}$ following
\begin{equation}
\bm{x'} = \bm{x} - \epsilon \text{sign} \nabla_{\bm{x}}\frac{\exp{(\bm{w}_{f(\bm{x})}^T \bm{z_x}} + b_{f(\bm{x})})}{\sum_{c \in \mathcal{C}} \exp{(\bm{w}_c^T \bm{z_x}} + b_c)}.
\end{equation}
Denote the penultimate layer feature of the noisy sample $\bm{x'}$ as $\bm{h'}$, ODIN assigns OOD score following: 

\begin{equation}
\frac{\exp{((\bm{w}_c^T \bm{h'}} + b_{c})/T)}{\sum_{c' \in \mathcal{C}} \exp{((\bm{w}_c'^T \bm{h'}} + b_{c'})/T)},
\end{equation}

\noindent where $c$ is the predicted class for the perturbed sample and $T$ is the temperature. 
ODIN is a hyperparameter-dependent algorithm and requires additional computation and dataset for hyper-parameter tuning. 
In our implementation, we set the noise magnitude as 0.0014 and the temperature as 1000. 

The computational complexity of ODIN is architecture-dependent. 
This is because the step of constructing the adversarial example requires back-propagation through the NN, whereas the step of evaluating the softmax score from the adversarial example requires an additional forward pass. 
Both steps require accessing the whole NN, which incurs significantly higher computational cost than our \fdbd~which only requires accessing the penultimate NN layer. 

\textbf{Energy}
\citet{liu2020energy} design an energy-based score function over the logit output.
Given a test sample $\bm{x}$, the energy based detection score can be represented as: 

\begin{equation} 
- \log \sum_{c \in \mathcal{C}} \exp{(\bm{w}_c^T \bm{z_x}} + b_c),
\end{equation}
where $\bm{z_x}$ is the penultimate layer embedding of $\bm{x}$.
The computational complexity of Energy on top of the inference process is $\Omega(|\mathcal{C}|T(\texttt{exp})+T(\texttt{log}))$, whereas $T(\texttt{exp})$ and $T(\texttt{log})$ are the computational complexity functions for evaluating the exponential and logarithm functions respectively. 
Note that the on-device implementation of exponential functions and the logarithm functions often requires huge look-up tables, incurring significant delay and storage overhead.

\textbf{ReAct}
\citet{sun2021react} build upon the energy score proposed by \citet{liu2020energy} and regularizes the score by truncating the penultimate layer estimation. 
We set the truncation threshold at $90$ percentile in our experiments.  


\textbf{ASH} \citet{djurisic2022extremely} build upon the energy score proposed by \citet{liu2020energy}. 
Prior to computing the Energy score, ASH sorts each feature to find the top-k elements, scales the top-k elements, and sets the rest to zero. 
We note that in addition to the cost of Energy, ASH introduces a sorting cost of $O(P \log k)$, where $P$ is the penultimate layer dimension. 

\textbf{Scale} \citet{xu2023scaling} build upon the energy score proposed by \citet{liu2020energy}. 
Prior to the Energy score, Scale sorts each feature to find the top-k elements. Based on the ratio between the sum of top-k elements and the sum of all elements, \citet{xu2023scaling} scale all elements in the feature. 
We note that in addition to the cost of Energy, Scale also introduces a sorting cost of $O(P \log k)$, where $P$ is the penultimate layer dimension.

\textbf{MDS}
On the feature space, \citet{lee2018simple} model the ID feature distribution as multivariate Gaussian and designs a Mahalanobis distance-based score: 

\begin{equation}\label{eq:maha_score}
\max_c - (\bm{e_x} - \hat{\bm{\mu}}_c)^T \hat{\Sigma}^{-1} (\bm{e_x} - \hat{\bm{\mu}}_c),
\end{equation}
where $\bm{e_x}$ is the feature embedding of $\bm{x}$ in a specific layer, $\hat{\mu}_c$ is the feature mean for class $c$ estimated on the training set, and $\hat{\Sigma}$ is the covariance matrix estimated over all classes on the training set.
Computing Eqn.~\eqref{eq:maha_score} requires inverting the covariance matrix $\hat{\Sigma}$ prior to inference, which can be computationally expensive in high dimensions. 
During inference, computing Eqn.~\eqref{eq:maha_score} for each sample takes $O(|\mathcal{C}|P^2)$, where $P$ is the dimension of the feature space. 
This indicates that the computational cost of MDS significantly grows for large-scale OOD detection.

On top of the basic score, \citet{lee2018simple} also propose two techniques to enhance the OOD detection performance.
The first is to inject noise into samples.
The second is to learn a logistic regressor to combine scores across layers. 
We tune the noise magnitude and learn the logistic regressor on an adversarial constructed OOD dataset, which incurs additional computational overhead. 
The selected noise magnitude in our experiments is 0.005. 

\textbf{CSI}
\citet{tack2020csi} propose an OOD-specific contrastive learning algorithm. 
In addition, \citet{tack2020csi} defines detection functions on top of the learned representation, combining two aspects: 
(1) the cosine similarity between the test sample embedding to the nearest training sample embedding 
and (2) the norm of the test sample embedding. 
As CSI requires specific training, which incurs non-tractible computational costs, we skip the computational complexity analysis for CSI here.

\textbf{SSD}
Similar to \citet{lee2018simple}, \citet{sehwag2020ssd} design a Mahalanobis-based score under the representation learning scheme.
In specific, \citet{sehwag2020ssd} propose a cluster-conditioned score:
\begin{equation}\label{eq:ssd_score}
\max_m - (\bm{e_x}/|\bm{e_x}| - \hat{\mu}_m)^T \hat{\Sigma}_m^{-1} (\bm{e_x}/|\bm{e_x}| - \hat{\mu}_m),
\end{equation}
where $\bm{e_x}/|\bm{e_x}|$ is the normalized feature embedding of $\bm{x}$ and $m$ corresponds to the cluster constructed from the training statistics. 

Computing Eqn.~\eqref{eq:ssd_score} requires inverting $m$ number of covariance matrix $\hat{\Sigma}_m$ prior to inference, which can be computationally expensive in high dimension. 
During inference, computing Eqn.~\eqref{eq:ssd_score} for each sample takes $O(|\mathcal{M}|P^2)$, where $|\mathcal{M}|$ is the number of clusters constructed in the algorithm and $P$ is the dimension of the feature space. 
This indicates that the computational cost of MDS significantly grows for large-scale OOD detection problems.

\textbf{KNN} \citet{sun2022out} propose to detect OOD based on the k-th nearest neighbor distance between the normalized features of the test sample $\bm{z_x}/|\bm{z_x}|$ and the normalized training features on the penultimate space. 
\citet{sun2022out} also observe that contrastive learning helps improve OOD detection effectiveness. 

In terms of computational complexity, normalizing the features is an $O(P)$ operation, where $P$ is the embedding dimension.
Computing the Euclidean distance between the normalized test feature and $N$ training features is an $O(NP)$ operation.
Additionally, searching for the $k_{th}$ nearest distance out of $N$ computed distances is a $O(N\log(N))$ operation.
Therefore, the overall inference complexity of KNN is $O(NP + N\log(N))$. 
Comparing to our $O(P + |\mathcal{C}|)$ algorithm \fdbd~, KNN exhibits much lower scalability for large-scale OOD detection, especially when the number of training samples $N$ significantly surpasses the number of classes $|\mathcal{C}|$. 

\textbf{ViM}
\citet{wang2022vim} propose to integrate class-specific information into feature space information by adding energy score to the feature norm in the residual space of training feature matrix. 
The detection score is designed to be:
\begin{equation}
    \alpha \sqrt{\bm{x}^T \bm{R} \bm{R} \bm{x}},
\end{equation}
where $\bm{R} \in R^{P\times(P-D)}$ correspond to the residual after subtracting the $D-$dimensional principle space.
In the preparation stage, ViM requires evaluating the residual/null space from the training data, which is computationally expensive given the data volume. 
During inference, large matrix multiplication is required, resulting in a computational complexity of $O((P - D)^2)$.

\section{Quantitative Study of the Proposed Distance Measuring method}\label{app:measuring_method}

In Section~\ref{sec:dist-measure}, we propose a closed-form estimation for measuring the feature distance to decision boundaries. 
To quantitatively understand the effectiveness and efficiency of our proposed method, we compare our method against measuring the distance via iterative optimization. 
In particular, we use targeted CW $L2$ attacks \cite{carlini2017towards} on feature space which can effectively construct an adversarial example which is classified into the target class from an iterative process.  
Empirically, CW attack-based estimation and our closed-form estimation differ by $<1.5\%$.
This implies that our closed-form estimation differs from the true value by $<1.5\%$, since estimation from a CW-attack upper bounds the distance whereas our closed-form estimation lower bounds the distance. 

We follow the Pytorch implementation of CW attacks proposed by \citet{papernot2018cleverhans} with the default parameters: initial constant 2, learning rate 0.005, max iteration 500, and binary search step 3. 
In our experiments, CW-attack has a success rate close to $100\%$.
On a Tesla T4 GPU, estimating the distance using CW attack takes 992.2ms per image per class.
In contrast, our proposed method incurs negligible overhead in inference, significantly reducing the computational cost of measuring the distance.

\section{Ablation Under FPR95}\label{app:fpr95_ablation}

In addition to the AUROC score reported in the main paper, we report our ablation study here under FPR95, the false positive rate of OOD samples when the true positive rate of ID samples is at $95\%$.
In Table~\ref{tab:ablation-fpr}, we compare the performance of OOD detection using the regularized average distance $\mathtt{regDistDB}$, the regularization term $\|z - \mu_{train} \|_2$, as well as the un-regularized average distances $\mathtt{avgDistDB}$ as detection scores, respectively.
Experiments are conducted on a ResNet-50 trained under cross-entropy loss following detailed setups in Section~\ref{sec:imagenet}.
The results in FPR95 further validate the effectiveness of regularization in our OOD detector. 

\begin{table}[H]
\vspace{-3mm}
\caption{\textbf{Regularization enhances the effectiveness of OOD detection.} 
FPR95 scores reported on ImageNet Benchmarks (lower is better). $\mathtt{regDistDB}$ outperforms $\mathtt{avgDistDB}$.
}
\label{tab:ablation-fpr}
\footnotesize
\centering
\includegraphics[width=0.55\columnwidth]{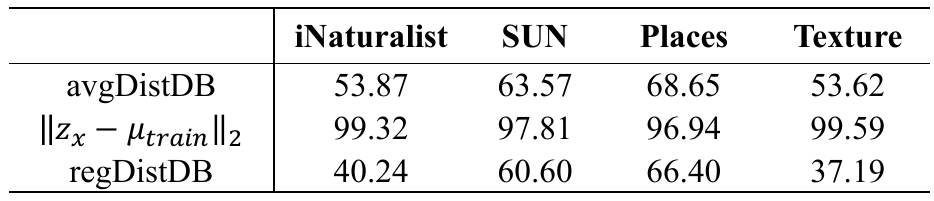}
\end{table}

\section{Feature Distances to Decision Boundaries}\label{app:individual}

We extensively validate our hypothesis that ID features tend to reside further away from decision boundaries than OOD features in Figure~\ref{fig:cifar-resnet}, Figure~\ref{fig:cifar-densenet}, and Figure~\ref{fig:imagenet-resnet}. 
To observe at a finer level of granularity, we sort per feature the estimated distances to all decision boundaries. 
On each subplot for a CIFAR-10 classifier, we plot 9 histograms, corresponding to the nearest distances, second nearest distance, and so on, up to the furthest distances. 
On each subplot for an ImageNet classifier, we sort the distance and plot every 100 ranked distances. 
We observe that ID features tend to reside further away from the decision boundaries compared to OOD samples across architectures and classification tasks.

\begin{figure*}[h]
\vspace{3mm}
\begin{center}
\centerline{\includegraphics[width=0.95\textwidth]{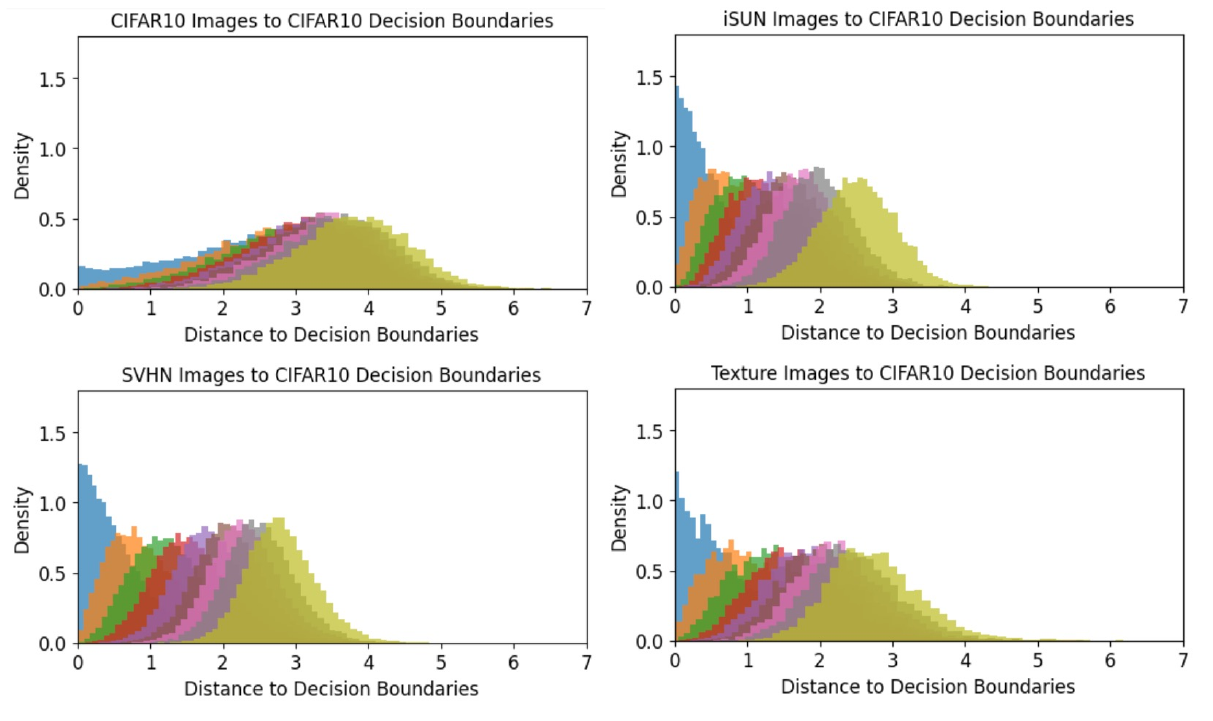}}
\end{center} 
\vspace{-5mm}
\caption{
Feature Distances to Decision Boundaries on a \textbf{ResNet-18 CIFAR-10} Classifier. 
ID features tend to be further away from the decision boundaries compared to OOD features. 
}\label{fig:cifar-resnet}
\end{figure*}

\begin{figure*}[t]
\begin{center}
\centerline{\includegraphics[width=0.95\textwidth]{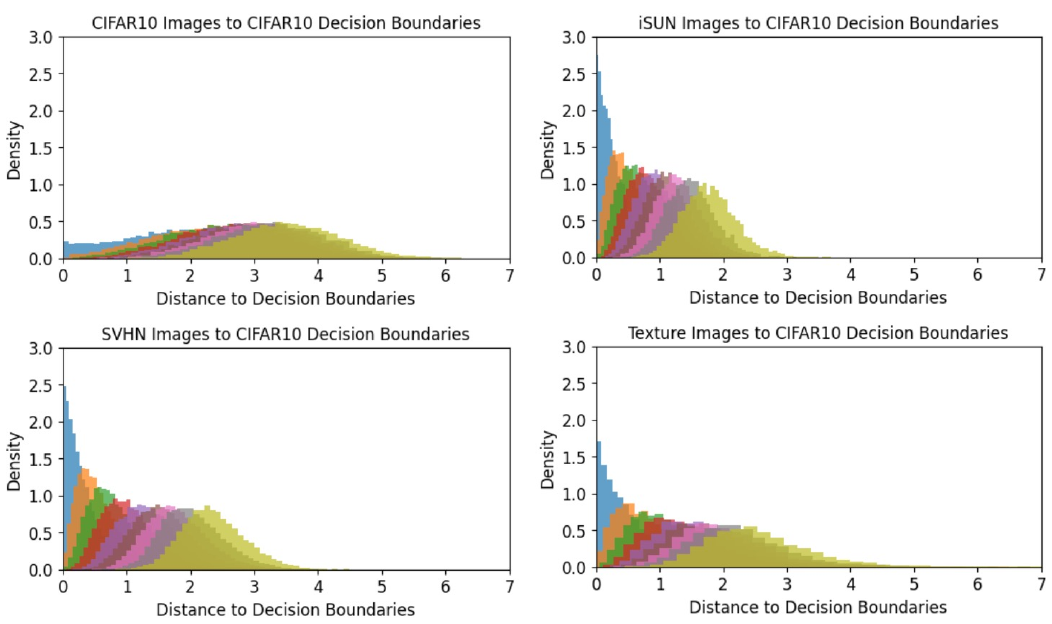}}
\end{center} 
\vspace{-5mm}
\caption{
Feature Distances to Decision Boundaries on a \textbf{DenseNet CIFAR-10} Classifier. ID features tend to be further away from the decision boundaries compared to OOD features. 
}\label{fig:cifar-densenet}
\vspace{-5.5mm}
\end{figure*}

\begin{figure*}[t]
\vspace{-3mm}
\begin{center}
\centerline{\includegraphics[width=\textwidth]{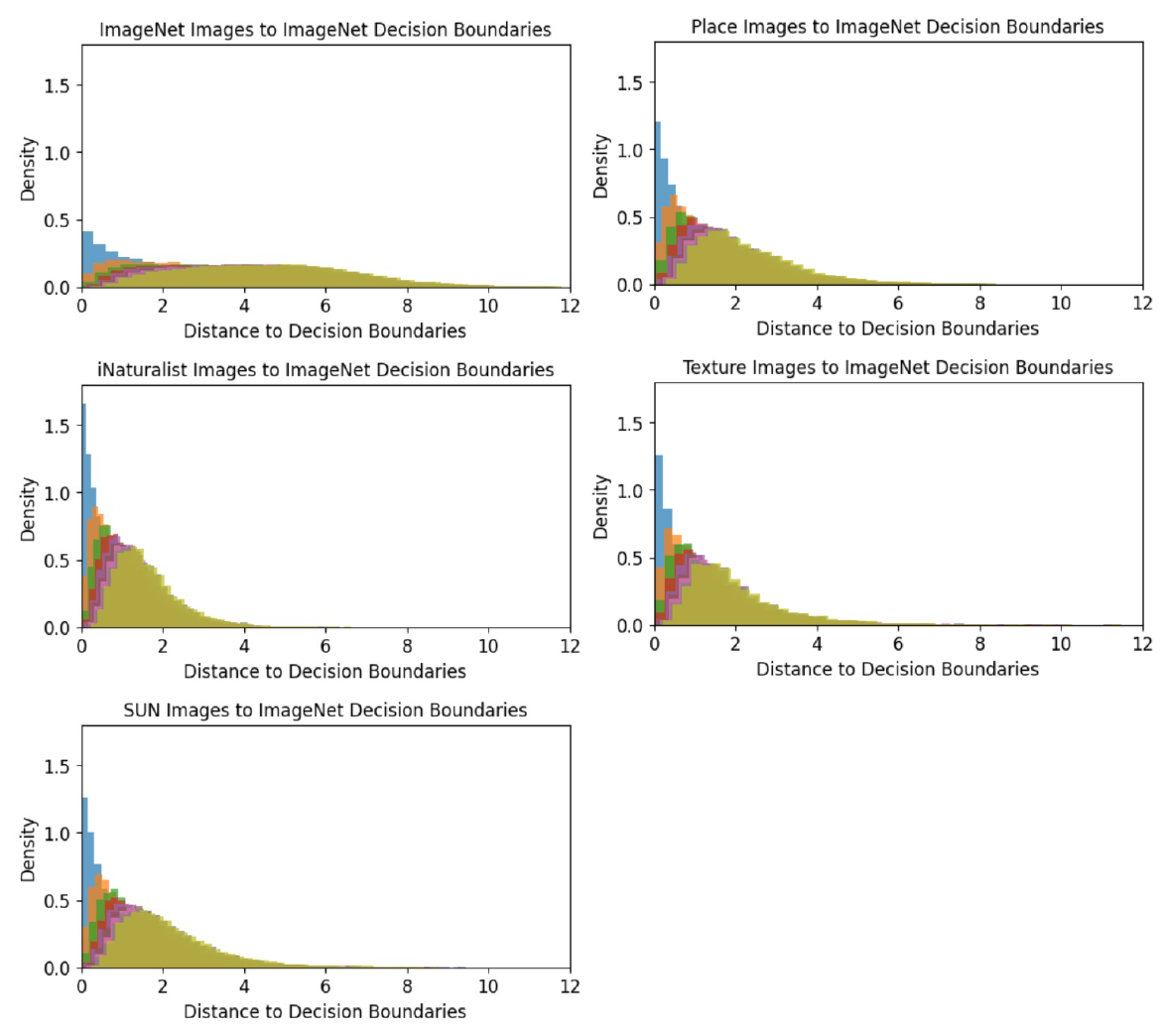}}
\end{center} 
\vspace{-5mm}
\caption{
Feature Distances to Decision Boundaries on a \textbf{ResNet-50 ImageNet} Classifier. ID features tend to be further away from the decision boundaries compared to OOD features. 
}\label{fig:imagenet-resnet}
\vspace{-5.5mm}
\end{figure*}


\end{document}